\documentclass[final,12pt]{msml2021} 

\usepackage{tikz}
\usetikzlibrary{cd}
\usepackage{booktabs}       
\usepackage{nicefrac}       
\usepackage{microtype}      
\usetikzlibrary{matrix}

\makeatletter
 \let\Ginclude@graphics\@org@Ginclude@graphics
\makeatother

\newcommand{\obar }{\overline}

\usepackage[verbose]{wrapfig}

\def\FFw{\edit{\FF_{\mathcal W}}}

\usepackage{aageneral}
\usepackage{aamath}
\usepackage{mhequ}
\usepackage{enumitem}

\mathtoolsset{showonlyrefs=true}

\usepackage{todonotes}

\newcommand{\hatt}{{ }}

\newcommand{\sigpsi}{\sigma_{\text{\scriptsize min}}^{D \psi^{-1}}}
\newcommand{\sigg}{\sigma_{\text{\scriptsize min}}^{g}}
\def\TD{(\mathcal T \mathcal D)}

\newcommand{\wii}{w^{(i)}}

\newcommand{\nnu}{{\tilde \nu}}

\DeclareMathOperator*{\argmin}{arg\,min}

\makeatletter
\newtheorem*{rep@theorem}{\rep@title}
\newcommand{\newreptheorem}[2]{%
\newenvironment{rep#1}[1]{%
 \def\rep@title{#2 \ref{##1}}%
 \begin{rep@theorem}}%
 {\end{rep@theorem}}}
\makeatother

\newreptheorem{theorem}{Theorem}
\newreptheorem{lemma}{Lemma}
\newreptheorem{proposition}{Proposition}

\newcommand{\edit}[1]{#1}

\usepackage{xr}

\renewcommand{\citet}{\cite}

\title[TD learning with nonlinear function approximation]{{Temporal-difference learning with nonlinear function approximation: lazy training and mean field regimes}}
\usepackage{times}

\msmlauthor{%
 \Name{Andrea Agazzi} \Email{agazzi@math.duke.edu}\\
 \addr   {Department of Mathematics\\
   Duke University\\
   Durham, NC 27708 }
 \AND
 \Name{Jianfeng Lu} \Email{jianfeng@math.duke.edu}\\
 \addr {Department of Mathematics\\
 Department of Physics and Department of Chemistry \\
 Duke University\\
 Durham, NC 27708}
}

\begin{document}


\maketitle

\begin{abstract}
  We discuss the approximation of the value function for
  infinite-horizon discounted Markov Reward Processes (\abbr{MRP}) with
  \aa{wide neural networks} trained with the Temporal-Difference (\abbr{TD}) learning
  algorithm. We first consider this problem under a certain scaling of the
  approximating function, leading to a regime called \emph{lazy} training. \aa{In
  this regime, which arises
  naturally, implicit in the initialization of the neural network, the parameters of the model vary only slightly during
  the learning process, resulting in approximately linear behavior of the model.}
   Both
  in the under- and over-parametrized frameworks, we prove exponential
  convergence to local, respectively global minimizers of the \abbr{TD} learning
  algorithm in the lazy training regime. We then compare the above scaling  with the alternative \emph{mean-field} scaling, where the approximately
  linear behavior of the model is lost. In this nonlinear, mean-field regime we prove
 that all fixed points of the dynamics in parameter space are global minimizers.
  We finally give examples of our convergence results in the case of models that diverge if trained
  with non-lazy \abbr{TD} learning.

\end{abstract}


\section{Introduction}

In recent years, deep reinforcement learning has pushed the boundaries of
Artificial Intelligence to an unprecedented level, achieving what was
expected to be possible only in a decade and outperforming human
intelligence in a number of highly complex tasks.  Paramount examples
of this potential have appeared over the past few years, with such
algorithms mastering games and tasks of increasing complexity, from
playing Atari to learning to walk and beating world grandmasters
at the game of Go
\citep{mnih:13,mnih:15,Silver:16,Silver:17,Silver:18,Haarnoja:18}.
Such
impressive success would be impossible without using neural networks
to approximate value functions and / or policy functions in
reinforcement learning algorithms. While neural networks, in
particular deep neural networks, provide a powerful and versatile tool
to approximate high dimensional functions
\citep{Cybenko:89,Hornik:91,Barron:93}, their intrinsic nonlinearity
might also lead to trouble in training, in particular in the context
of reinforcement learning. For example, it is well known that
nonlinear approximation of the value function might cause divergence in
classical temporal-difference learning due to instability
\citep{tsitsiklis}. While several algorithms have been proposed in the
literature to address the issue of non-convergence  {of nonlinear approximators (\eg Gradient Temporal Difference \citep{sutton09c}, GTD2, TD with gradient correction \citep{sutton09b} and many others
\citep{Riedmiller:2005,sutton09a,Maei:10,Szepesvari:10})},  practical deep
reinforcement learning often employs and prefers basic algorithms such
as temporal-difference \citep{Sutton:88} and Q-learning
\citep{Watkins:89} due to their simplicity. It is thus crucial to
understand the convergence of such algorithms and to bridge the gap
between theory and practice.

The theoretical understanding of deep reinforcement learning is of
course rather challenging, as even for supervised learning, which can
be viewed as a special case of reinforcement learning, deep neural
networks are still far from being understood despite the huge amount
of research focus in recent years. On the other hand, recent progress
has led to an emerging theory for neural network learning
including recent works on the
mean-field point of view of training dynamics
\citep{MeiMonNgu18,RotVE18,RotJelBruVE19,WeiLeeLiuMa18,ChizatBach18, stephan20}
and also on the linearized training dynamics in the over-parametrized
regime
\citep{hongler18,AZLiSong18,DuZhaiPocSin18,DuLeeLiWangZhai:18,ZouCaoZhouGu, AZLiLiang18,cb182,OymSol19,Montanari:19b,Lee19}.

The main goal of this work is to analyze the dynamics of a
prototypical reinforcement learning algorithm --
Temporal-Difference (\abbr{TD}) learning -- based on the recent progress in
deep supervised learning.
In particular, we first focus on the lazy
training regime,  {(also known as the neural tangent kernel regime when training wide neural networks)}, and analyze
\abbr{TD} learning in both over-parametrized and
under-parametrized regimes with scaled value function approximations.
 {We then compare such lazy models with their mean-field counterpart in
terms of accuracy and convergence for \abbr{TD} learning.}

\paragraph{Related Works.}

This work is closely related to the  recent paper \citet{cb182},
addressing the problem of lazy training in the supervised learning
framework when models are trained through (stochastic) gradient
descent. In particular, that paper  introduced the scaling that we
consider in this work as an explanation, \eg of the small relative
displacement of the weights of over- and under-parametrized neural
networks for supervised learning. That work, however, leverages the gradient
structure of the underlying vector field, which we lack in the present framework
when the underlying policy is not reversible \citep{ollivier18}.
{The linear stability analysis is also considered in
the recent work \citet{Achiam19} based on the neural tangent kernel
\citep{hongler18} for off-policy deep Q-learning.}

The groundbreaking paper \citet{tsitsiklis} proves convergence of \abbr{TD}
learning for linear value function approximation, unifying the
manifold interpretations of this convergence phenomenon that preceded
it by highlighting that convergence of the algorithm is to be
understood in the norm induced by the invariant measure of the
underlying Markov process. Furthermore, the paper gives an
illuminating counterexample for the extension of the linear result to
the general, nonlinear setting. Our result shows that divergence does not occur in the lazy training regime.

{
The recent work
\citet{BrandfonbrenerBruna} discusses the stabilizing effect homogeneity in the approximating function class on \abbr{td} training dynamics. This work further shows convergence and non-divergence of \abbr{td} learning in
the over-parameterized, respectively the under-parametrized regime,
provided that the environment is sufficiently
reversible.
We note that working in the lazy training regime allows to ensure convergence independently on the reversibility of the environment and to quantify
the error of the fitted model in the under-parametrized regime.
Another recent work \citet{CaiYangLeeWang} analyzes global convergence of a
modified \abbr{td} algorithm for two-layer neural networks with ReLu nonlinearity
when the width of the hidden layer diverges.
In contrast, in the present paper we focus on the \emph{original} TD($\lambda$) learning algorithm, and consider various training regimes. Furthermore, our convergence results in the lazy training regime apply to more general approximators (including, but not limited to, multiple activation functions). More recently, \citet{wang20td} discuss the convergence of \abbr{td} learning with neural network representations in the mean-field regime. The optimality bounds established in that paper, however, critically depend on a scaling parameter $\alpha$ which is assumed to be large. This scaling parameter corresponds to the lazy training parameter discussed in this paper. Therefore, the analysis of \cite{wang20td} can be seen as a combination of the limits taken in this work (large width first and lazy training after). We also note that the results of \cite{wang20td} are restricted to neural networks with bounded weights in the last layer. Further recent works discussing convergence properties of deep reinforcement learning algorithms both in the lazy and the mean-field regime include \citep{agazzi20,wang19pg,wang19adv,wang20single,wang19lqr,spiliopoulos19}.

 {Finally, the mean-field analysis in our paper is tightly connected to the recent line of work \citep{MeiMonNgu18,RotVE18,RotJelBruVE19,WeiLeeLiuMa18,ChizatBach18}. In particular, our proof of global optimality of the fixed points extends
the results from \cite{ChizatBach18} bypassing the lack of the gradient structure in the \abbr{TD} learning setting.}

}

\paragraph{Contributions.}
This paper discusses the use of wide neural networks as nonlinear function approximators in value-based reinforcement learning. In particular, we consider on-policy \abbr{TD} learning for policy evaluation, a widely used algorithm for value function approximation tasks. We study convergence and optimality of wide neural networks trained with this algorithm under different scaling regimes of the parameters at their initialization, contrasting the results to highlight advantages and drawbacks of the various choices of initialization. More precisely:
\begin{enumerate}
  \item We prove convergence of \abbr{TD} learning (asymptotically with probability one) in the lazy training regime.
  More specifically, we prove convergence of this algorithm in both the under- and over-parametrized regime to local and global minima, respectively, of a natural, weighted error function (the projected \abbr{td} error), and illustrate such convergence properties through numerical examples.

\item To obtain the result summarized above, we adapt the contraction conditions developed in the framework of linear function approximations to a nonlinear, differential geometric setting.
Furthermore, we extend existing results on the convergence in the lazy training regime of nonlinear models trained by gradient descent in the supervised learning framework to the context of reinforcement learning. This requires a generalization of the techniques developed in the gradient flow setting to non-gradient (\ie rotational) vector fields such as the ones encountered in the \abbr{TD} learning framework.

\item To put the result in 1. into perspective, we
compare the lazy training regime to the alternative mean-field regime. In particular we show that \aa{under} some technical assumptions, in the mean-field case all the stationary points of the \abbr{TD} dynamics are global minimizers, \ie the model perfectly approximates the desired value function. This result provides evidence that models in the mean-field regime benefit of a far stronger approximating power.
\end{enumerate}
\paragraph{Organization of the paper} The paper is organized as follows: In \sref{s:notation} we introduce the framework of Markov reward processes and the \abbr{TD} learning algorithm, discussing the approximations made throughout the paper and introducing the lazy training regime. In \sref{s:over-parametrized} we discuss the convergence properties of the lazy training regime in the over-parametrized setting. In \sref{s:under-parametrized} extend the convergence results in the lazy setting to the under-parametrized case, and we compare this regime with the mean-field regime. We give some numerical examples in \sref{s:numerical} and discuss our results in \sref{s:conclusion}. Technical proof are deferred to the appendix.

\section{Markov Reward Processes, TD learning and scaling limits} \label{s:notation}
We denote a Markov Reward Process (\abbr{MRP}) by the $4$-tuple $(\SS,P,r,\gamma)$, where $\SS$ is the state space, $P = P(s,s')_{s,s' \in \SS}$ a transition kernel, $r(s,s')_{s,s' \in \SS}$ is the real-valued, bounded immediate reward function and $\gamma \in (0,1)$ is a discount factor. In this context, the \emph{value function} $V~:~\SS \to \Rr_+$ maps each state to the infinite-horizon, expected discounted reward obtained by following the Markov process defined by $P$. We assume
that this Markov process satisfies the following assumption:
\begin{assumption}\label{a:P}
The Markov process with transition kernel $P$ is ergodic and its stationary measure $\mu$ has full support in $\SS$. {Furthermore we assume that $\SS$ is compact.}
\end{assumption}

We are interested in learning the value (or cost-to-go) function $V^*(s)$ of a given \abbr{mrp} $(\SS,P,r,\gamma)$, which is given by
\begin{equ}\label{e:extremum}
  V^*(s) := \Ex{s}{\sum_{k = 0}^\infty \gamma^k r(s_k,s_{k+1})},
  \end{equ}
  where $\Ex{s}{\,\cdot\,}$ denotes the expectation of the stochastic process $s_k$ starting at $s_0 = s$. More specifically we would like to estimate this function through a set of predictors $\hatt V_w(s)$ in a Hilbert space $\FF$  parametrized by a vector $w \in \WW := \Rr^p$. We make the following assumption on such predictors:
  \begin{assumption}\label{a:V}
    The parametric model $\hatt V~:~\Rr^p \to \FF$ mapping
    $w \mapsto \hatt V_w(\,\cdot\,)$ is differentiable with Lipschitz
    continuous derivative {$D\hatt V~:~w\mapsto DV_w$ (where $DV_w$ is a linear map from
    $\Rr^p \to \FF$)} with Lipschitz constant $L_{D\hatt V}$ defined
    \abbr{wrt} the operator norm.
  \end{assumption}
  A popular algorithm to solve this problem is given by value function approximation with
  \emph{TD($\lambda$)} updates \citep{SuttonBarto:18}. Starting from an initial condition
  $w(0) \in \WW$, for any $\lambda \in [0,1)$, this learning algorithm
  updates the parameters $w$ of the predictor \edit{along a path $\{s_k\}$ in $\mathcal S$} by the following rule:
\begin{equ}\label{e:tdu1}
w(t+1) := w(t) + \beta_t \,\delta(t,t) \,\edit{z_\lambda(t;\{s_k\}_{0}^t)}\,,
\end{equ}
for a \emph{fixed} sequence of time steps $\{\beta_t\}$ to be specified later, where the \emph{temporal-difference error} $\delta(\cdot,\cdot)$ and \emph{eligibility vector} $z_\lambda(\cdot;\cdot)$ are given by
\begin{equ}\label{e:tde1}
  \delta(t,k) := r(s_k,s_{k+1}) + \gamma \hatt V_{w(t)}(s_{k+1}) - \hatt V_{w(t)}(s_k)\, \qquad \edit{z_\lambda(t;\{s_k\}_{k_0}^{k_1}):= \sum_{\tau = k_0}^{k_1} (\gamma \lambda)^{k_1-\tau} \nabla_w \hatt V_{w(t)}(s_\tau)}\,.
\end{equ}

This work focuses on the asymptotic regime of small constant
step-sizes $\beta_t \to 0$. In this adiabatic limit, the stochastic
component of the dynamics is averaged out before the parameters of the
model can undergo a significant change.  This allows to consider the
\abbr{TD} update as a deterministic dynamical system emerging from the
averaging of the underlying stochastic algorithm. We focus on analysis
of this deterministic system to highlight the aspect of nonlinear
function
approximation. 
The averaged, deterministic dynamics is given by the set of
\abbr{ode}s
\begin{equ}\label{e:tduc}
  \frac{\d}{\d t} w(t) =  \Ex{\mu}{\bigl(r(s_0,s_1) + \gamma \hatt V_{w(t)}(s_1) - \hatt V_{w(t)}(s_0)\bigr) \edit{z_\lambda(t;\{s_k\}_{-\infty}^0)}}\,,
\end{equ}
where $\mathbb E_\mu$ denotes the \edit{expectation of the underlying dynamics at stationarity (started at $k_0 = - \infty)$ and $z_\lambda(t;\{s_k\}_{k_0}^{k_1})$ is defined in \eref{e:tde1}}. {In the case of finite state space ($|\SS| = d$) we can represent $\hatt V_w$ as a vector in $\Rr^d$, while in general it is a function $\SS \to \Rr$, which we will restrict to the \edit{Hilbert} space \edit{$\mathcal F := L^2(\SS, \mu)$} of square integrable functions \abbr{wrt} $\mu$.} 

To streamline our analysis, we define the \abbr{TD} operator $T^\lambda~:~L^2(\SS,\mu)\to L^2(\SS,\mu)$:
\begin{equ}
  T^\lambda V(s) := (1-\lambda) \sum_{m = 0}^\infty \lambda^m \Ex{s}{\sum_{k = 0}^m \gamma^k r(s_k,s_{k+1}) + \gamma^{m+1} V(s_{m+1})}\,.
\end{equ}
Note that when $\lambda = 0$ the above operator acquires the simple form
$T^{0} V := \bar r + \gamma P V$ for $\bar r(s) := \Ex{s}{r(s,s')}$. {Then, denoting throughout by $ D V_{w}$ the
Fréchet derivative of $V$ at $w$\footnote{\edit{in the finite-dimensional case $|\mathcal S| =d<\infty$, $D V_{w}$ can be identified with the $d \times p$-dimensional Jacobian matrix of $V_{w}$}},} it can be shown
\cite[Lemma 8]{tsitsiklis} (and is immediately verified
when $\lambda = 0$) that the dynamics \eqref{e:tduc} for general
$\lambda < 1$ can be written as
\begin{equ}\label{e:tduav}
  \frac{\d}{\d t} w(t) =  \dtp{T^\lambda \hatt V_{w(t)} - \hatt V_{w(t)},  D \hatt V_{w(t)}}_\mu \,,
\end{equ}
where we define throughout the inner product \edit{on $\mathcal F$} induced by the invariant measure $\mu$ {(acting componentwise in expressions such as the one above)} as
\begin{equ}\label{e:gamma}
  \dtp{a,b}_\mu := \int_\SS{a(s) b(s)} \mu(\d s)\,,
\end{equ}
and denote throughout by $\|\,\cdot \,\|_\mu$ and $\Gamma$ the corresponding norm and metric tensor respectively.
Note that in the case $|\SS|= d$, $\Gamma$ reduces to the $d$-dimensional diagonal matrix whose entries are the (positive) values of the invariant measure $\mu(s)$, and one has $\dtp{a,b}_\mu = a^\top \Gamma b$.

The extension of convergence results for the limiting, average dynamics we consider in this paper to convergence with probability one of the underlying, stochastic algorithm can be obtained through standard stochastic approximation arguments \citep{meyn00, Borkar:09}. More details on this extension are given in Remark~\ref{rem:stochapprox} in \sref{s:proofs} and in the appendix.

\subsection{Mean-field models and lazy training regime}

In this paper, we consider models of the form
\begin{equ}\label{e:meanfield}
 V_w(s) =  \frac1 N \sum_{i = 1}^N \psi(s; \wii) {\qquad \text{where } \wii \in \Omega \subseteq \Rr^m \,\forall \, i\in(1,\dots, N)\,, \psi:\SS\times \Omega \mapsto \FF}\,,
\end{equ}
where the prediction of the model is obtained by averaging the output of $N$ ``simple'', identical units $\psi$. In this case  the \abbr{td} update reads $\frac \d {\d t}\wii(t) = \edit{\dtp{T^\lambda V_{w(t)} - V_{w(t)}, N^{-1}\nabla_{\omega} \psi(\cdot;\wii)}_\mu}\,.$ This family of models includes single layer neural networks \citep{ChizatBach18}, radial basis functions and linear models \citep{RotVE18}.

We further introduce a certain scaling of the \abbr{TD} learning update, which will be the central object of study of this paper. More specifically, we define the rescaled update
\begin{equ}\label{e:tduavscaled}
  \frac{\d}{\d t} w(t) = \frac{1}{\alpha} \dtp{T^\lambda (\alpha \hatt V_{w(t)}) -\alpha\hatt V_{w(t)},  D \hatt V_{w(t)}}_\mu
\end{equ}
for a scaling parameter $\alpha \geq
1$\,. This update, \edit{resulting from a simultaneous rescaling of time and of $V_w$ in \eref{e:tduav},} was designed to capture the effective training dynamics of neural network models under various initializations of their parameters, as we detail below.
\begin{figure}[t]
 \centering
 \def\svgwidth{.8\textwidth}
 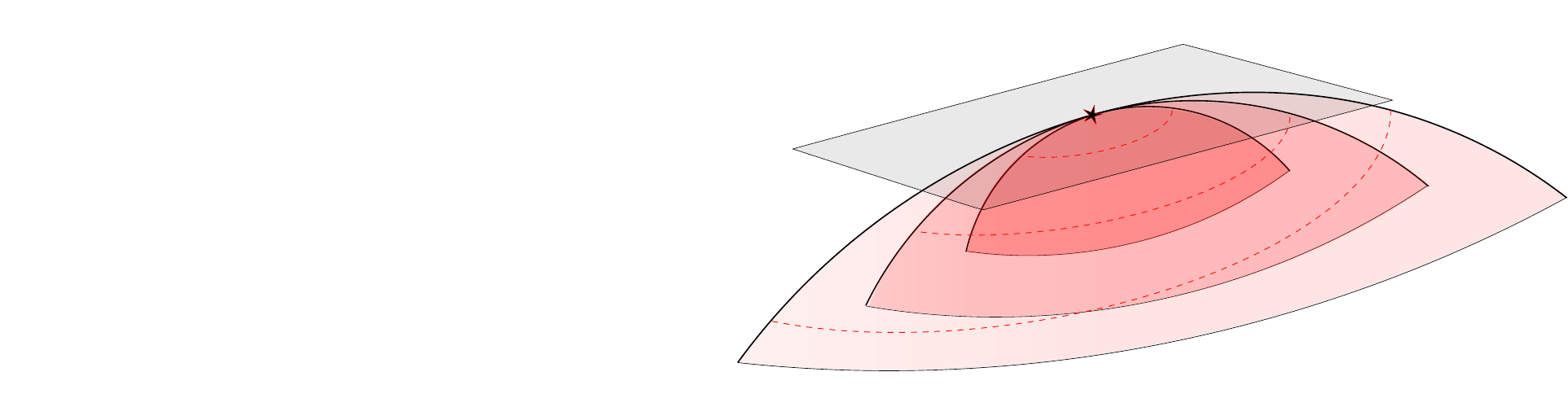
 \caption{Schematic representation of the effect of the linear scaling
   of the approximating function (\eg in \eref{e:vf}) for different values of $\alpha$ in the
   \emph{under-parametrized} setting. The space of parameters (left)
   is mapped to the space of predictors (right) by the parametric
   model $\hatt V$. The scaling $\hatt V \to \alpha \hatt V$ changes
   the manifold $\FFw$ that the parameter space is mapped to (different
   surfaces on the right). In particular, this scaling ``widens'' the
   reach in the space of functions of the predictors within a ball of
   small radius in $\WW$, but at the same time it ``flattens'' that
   space (locally in $\mathcal W$) bringing it closer to the
   tangential plane \edit{$\mathcal T_{\hatt V_{w(0)}} \FFw$ to $\FFw$ at} the initial model $\hatt V_{w(0)}$. Choosing
   $V_{w(0)} = 0$ as in the picture above leaves the initial point of
   the dynamics (in predictor space) invariant under such
   transformation.}
 \label{f:scaling}
\end{figure}

The remainder of the paper discusses  the large-$N$ limit of \eref{e:meanfield}, and contrast the effect of two possible choices of $\alpha$ when scaling such model as $V_w \to \alpha V_w$:
\begin{enumerate}
  \item{\bf Lazy training regime ($\alpha \to \infty$
  as $N \to \infty$)} This regime is realized for \emph{large} values of the scaling parameter $\alpha$ in \eref{e:tduavscaled} . One of the reasons why this scaling of the model is of practical
  interest is because it arises naturally when training neural networks,
  implicit in some widely applied choices of initial conditions, as we
  explain in \sref{s:nn} and as discussed in \eg \cite{cb182, CaiYangLeeWang}.
  As we shall see below, under some
  mild assumptions for large values of $\alpha$ the parameters $w$ of
  the model vary only slightly during training, justifying the name \emph{lazy training} regime.  A visual representation of the geometric
  effect of this scaling in the case where {$p < d< \infty$} is given in
  \fref{f:scaling}. When the parameters $\wii$ are drawn \abbr{iid} from a common distribution, as $\alpha \to \infty$ together with  $N \to \infty$ the model converges to a \emph{deterministic} kernel, called \emph{neural tangent kernel} (\abbr{NTK}) \citep{hongler18}. We investigate the effect of this scaling transformation on the training dynamics in Sections~\ref{s:over-parametrized} and \ref{s:lt2}.

  \item{\bf Mean-field regime ($\alpha = \mathcal O_N (1)$ and $N \to \infty$)} This regime is realized, contrarily to the above, for \emph{bounded} values of $\alpha$, for example fixing $\alpha = 1$ while $N \to \infty$, and arises under a different scaling of the model parameters at initialization as discussed in \sref{s:nn}. This setting was investigated in the context of supervised learning \cite{MeiMonNgu18, RotVE18, ChizatBach18} under the name of \emph{mean-field} limit. A heuristic justification that the scaling in \eref{e:meanfield} is natural for the model $V_w$ and does not lead to the lazy regime can be found \eg in \citet{cb182}.
  In \sref{s:mf} we investigate the optimality properties of neural networks in this regime and compare it to the lazy training from point 1.
\end{enumerate}

\label{s:proofs}
\section{Over-parametrized regime}\label{s:over-parametrized}
In the over-parametrized setting we assume that {$D\hatt V_{w(0)}$ is surjective, \ie its singular values are uniformly bounded away from $0$. This is only possible in the finite state space setting and is automatically the case if} the number of parameters $p$ is larger than the size of the state space $\SS$.
Admittedly, in applications such as AlphaGo \citep{Silver:16, Silver:17}, it is unrealistic to
  over-parametrize, but we start with this regime,  which parallels the study
  of over-parametrized supervised learning,
   to develop the intuition of the reader and introduce some important lemmas. Analysis of the under-parametrized regime will be
  discussed in the next section.
In
order to state our first result, we introduce the scalar product in
{$\FF$ defined by}
  $\dtp{a,b}_0 = \dtp{a, g_{w(0)} b} $ where $g_w := (D\hatt V_{w} \cdot D\hatt V_{w}^\top)^{-1}$,
and denote by $\|\,\cdot\,\|_0$ the norm it induces. Note that $g_w$ is the metric tensor of the pushforward metric induced by the parametric model $\hatt V~:~\Rr^p \to \FF$. Furthermore, we note that if the singular values of $D\hatt V_{w(0)}$ are uniformly bounded away from $0$, the norms $\|\,\cdot \,\|_\mu, \|\,\cdot \,\|_0$ are equivalent, \ie there exists $\kappa > 0$ such that $\kappa^{-1}\|f\|_0<\|f\|_\mu < \kappa \|f\|_0$ for all $f \in \FF$\,.

\begin{theorem}[Lazy training, over-parametrized case]\label{t:over-parametrized}
  Assume that $\sigma_{\text{min}}>0$, where $\sigma_{\text{min}}$ is
  the smallest singular value of $D\hatt V_{w(0)}$. Assume further that $w(0)$ is such that
  $\|\hatt V_{w(0)}\|_0 < M := ({1-\gamma})^2\sigma_{\text{min}}^2/(192
  \kappa^2 L_{D \hatt V}\|D\hatt V_{w(0)}\|)$, then for
  $\alpha > \alpha_0 := \|V^*\|_0/M$ we have for all $t \geq 0$ that
\begin{equ}\label{e:expcontraction}
  \|V^*- \alpha \hatt V_{w(t)}\|_0^2 \leq \|V^*- \alpha \hatt V_{w(0)}\|_0^2 e^{- \frac{1-\gamma}{2\kappa^2 } t }\,.
\end{equ}
Recall that $V^*$ is the exact value function given by \eref{e:extremum}. Moreover, if $\|\hatt V_{w(0)}\|_0 \leq C \alpha^{-1}$ for a constant $C > 0$, then $\sup_{t>0}\|w(t)- w(0)\| = \mathcal O(\alpha^{-1})$.
\end{theorem}

\begin{proof}\noindent\textbf{sketch of \tref{t:over-parametrized}} Similarly to
the proof in
\citet{cb182}, we first show that $D V_w$ and $V_w$ do not change much assuming that $w$ stays in a small ball of radius $\rho$.
Then, combining this result with the Lipschitz continuous character of $D V$ in $w$, we show that $w$ does indeed stay in the desired ball of radius $\rho$. 
To bypass the absence of a strongly convex cost functional 
in our framework, which was crucial in the analysis of \citet{cb182}, we adopt a strategy based on the use of a local Lyapunov function
\begin{equ}\label{e:lf}
  U(f) = \|f - V^*\|_0^2\,,
\end{equ}
where $V^*$ is the sought for value function \eref{e:extremum}. The theorem is based on two preparatory lemmas.
The first one (\lref{l:perturbation}) states that for large
values of the scaling parameter $\alpha$ the pushforward metric $g_w$
varies in a negligible way during training.
To state the lemma, we denote throughout by
$\indicator$ the identity map in the corresponding space and by
$\BB_\rho^{\mu}(v)$, $\BB_\rho^{0}(v)$ and $\BB_\rho(v)$ the balls with
radius $\rho$ around $v$ in $\|\,\cdot\,\|_\mu$, $\|\,\cdot\,\|_0$
and $\|\,\cdot\,\|_2$ respectively.
\begin{lemma}[Perturbation of the metric] \label{l:perturbation}
    Let $\GG_0$ be a compact subset of a linear space $\GG$. For {$v(0) \in \GG_0$}, let $g_v$ be a continuous, {self-adjoint linear operator} that is positive definite in a neighborhood of $v(0)$ when restricted on $\GG$. Then for all $\epsilon > 0$ there exists $\delta > 0$ such that, for all $v \in \BB_\delta(v(0)) \subseteq \GG_0$
  \begin{equ}\label{e:perturbation}
    g_{v(0)} =  (\indicator + \tilde  g_v)g_v\,,
  \end{equ} for a linear operator $\tilde g_v$ with
  $  \|\tilde g_v\|< \epsilon\,.$
  More specifically, let $\sigma_{\text{min}}$ be the smallest singular value of $DV_{w(0)}$. Then if $\rho \leq ({1-\gamma})\sigma_{\text{min}}^2/(48 L_{D \hatt V})$,
   \eref{e:perturbation} holds
  with $\|\tilde g_{V(w)}\|< \frac{1-\gamma}4$ for all $w \in \BB_{\rho}(w(0))$.
\end{lemma}
We also recall from \cite{tsitsiklis} the following contraction property of
the \abbr{td} operator in the $\|\,\cdot\,\|_\mu$ norm.
For the convenience of readers, we recall the proof in the
appendix.
\begin{lemma}\cite[Lemmas 1, 3, 7]{tsitsiklis}\label{l:tsitsiklis}
Under \aref{a:P}, for any $V, \tilde V \in \FF$ we have that
$
  \|T^{\lambda} V- T^{\lambda} \tilde V\|_\mu \leq \gamma_\lambda \| V- \tilde V\|_\mu$\,\, for $ \gamma_\lambda := \gamma\frac{ 1-\lambda}{1-\gamma \lambda} \leq \gamma< 1\,.
$
In particular there exists a unique fixed point of $T^\lambda$, $V^* \in \FF$ given by \eref{e:extremum}.
\end{lemma}

To prove \tref{t:over-parametrized}, we then rely on the contraction $T^\lambda$ (\lref{l:tsitsiklis}, from \cite{tsitsiklis}) to
establish decay of the local Lyapunov function $U$ as long as $w$
stays within a ball. furthermore, by \lref{l:perturbation} the nonlinear effects become negligible when
$\alpha$ is sufficiently large. The control of $U$ in turn gives the
bound of the change of $w$, which closes the argument. The full proofs of the theorem and the lemmas are
given in the appendix.
\end{proof}

\begin{remark}\label{rem:stochapprox}
  Our results can be extended to show stability and convergence in the
  stochastic approximation setting, similarly to \citet{tsitsiklis,
    sutton09a}, under the additional assumption that the step size
  $\{\beta_t\}$ satisfies the Robbins-Monro condition
  \citep{RobbinsMonro:51}. For example, one can apply \cite[Thms. 2.2,
  2.4]{meyn00} guaranteeing almost sure convergence and exponential
  contraction of the expected error with probability one over the
  initial condition provided that 
  the limiting vector field (in our case \eref{e:tduavscaled}) has a
  unique fixed point and is Lipschitz continuous. Lipschitz continuity
  is an immediate consequence of the {linearity of $T^\lambda$ and the boundedness of} closed balls in
  $\FF$ together with the Lipschitz continuity of the models
  \aref{a:V}. 
  The existence of a fixed point \eref{e:extremum} in $\FF$ of the
  limiting vector field is trivial while its uniqueness
  is shown in the proof of \tref{t:over-parametrized} in the
  appendix.
\end{remark}

\section{Under-parametrized regime}\label{s:under-parametrized}

The underlying assumption in this section is
that the size of state space is larger than the number of parameters,
which in turn bounds the rank $r$ of $D \hatt V_{w(0)}$ from above:
$ r < p < d$. {In particular, in the case of wide neural networks of interest in this paper, this assumption is realized when the state space is not finite and, in particular, in the case of continuous state space $\mathcal S$.}

\subsection{Lazy training regime}\label{s:lt2}
In this regime, in general, there is no hope that
\abbr{TD} will converge to the true value function $V^{\ast}$. In
fact, the image of the operator $T^{\lambda}$ might not even lie in
the space $\FFw$ of approximating functions. However, the derivative
$D \hatt V_{w(t)}^\top$ in the \abbr{TD} update acts as a projection
(\abbr{wrt} the product $\langle\,\cdot\,,\,\cdot\,\rangle_\mu$) onto
the tangent space of $\FFw$ at $\hatt V_{w(t)}$ (more specifically,
$D \hatt V_{w(t)}^\top$ projects the image of $T^\lambda$ onto $\WW$,
which is then mapped back by
$D \hatt V_{w(t)}$ to \edit{the tangent space} $\edit{{\mathcal T}_{\hatt V(w(t))}}\FFw$ \edit{of $\FFw$ at $\hatt V(w(t))$}). We denote throughout by
{$\Pi$
and
  $\Pi_0$
  } the projection
operator under \eref{e:gamma} onto $\edit{{\mathcal T}_{\hatt V(w(t))}}\FFw$ and
$\edit{{\mathcal T}_{\hatt V(w(0))}}\FFw$ respectively. What one can hope for is that
the \abbr{TD} algorithm converges to a locally ``optimal''
{approximation $\tilde V^{\ast}$ of $V^{\ast}$ on the manifold
  $\FFw$, which is close to the best approximator $\Pi_0 V^{\ast}$ of
  $V^{\ast}$ on the linear tangent space $\edit{{\mathcal T}_{\hatt V(w(0))}}\FFw$.}
\begin{theorem}[Lazy training, under-parametrized case]\label{t:under-parametrized}
  Assume that $r := \text{rank}(D \hatt V_w)$ is constant in a neighborhood of $w(0)$  and $V_{w(0)} = 0$.
Then there exists $\alpha_0 > 0$ such that for any $\alpha > \alpha_0$ the dynamics \eref{e:tduavscaled} (and the corresponding approximation $\hatt V_{w}$) converge exponentially fast to a locally (in $\WW$) attractive fixed point $ \tilde V^*$, for which $\|\Pi(T^\lambda \tilde V^* - \tilde V^\ast)\|_\mu = 0$ and $\|\tilde V^* - V^*\|_\mu < \frac{1-\lambda \gamma}{1 - \gamma} \|\Pi_0 V^* - V^*\|_\mu+\OO( \alpha^{-1})$.
\end{theorem}
{Note that for random initialization the constant rank assumption is satisfied {almost surely}. Indeed, the maximal rank property holds generically in $\mathcal W$ and thus {with probability $1$} at $w(0)$ when the model parameters are initialized randomly. Furthermore, by lower semicontinuity of the rank function the Jacobian $DV$ will have maximal rank in an \emph{open} subset of $\mathcal W$, {ensuring the existence of the required neighborhood.}} The proof of \tref{t:under-parametrized} (sketched below) is given in the appendix.

The main difference in the proof of \tref{t:under-parametrized} \abbr{wrt} \tref{t:over-parametrized}
  is that
$D \hatt V_{w} \cdot D \hatt V_{w}^\top$ does not have full rank
anymore. This implies on one hand that the norms $\|\,\cdot\,\|_\mu$
and $\|\,\cdot\,\|_0$ {are not equivalent in $\FF$}, even though we still have
$\|\,\cdot\,\|_0 \leq \kappa \|\,\cdot\,\|_\mu$ for a $\kappa > 0$,
{provided that \aref{a:P} holds.}  On the other hand, as
mentioned above, this implies that the model $\hatt V_w$ evolves on a
submanifold $\FFw$ of $\FF$, and that $T^\lambda$ does not, in general, map
onto the tangential plane $\edit{{\mathcal T}_{\hatt V(w)}}\FFw$ of $\FFw$ at $\hatt V_w$.
The action of $T^\lambda$ is then projected back onto $\edit{{\mathcal T}_{\hatt V(w)}}\FFw$ by the operator $D V_{w(t)}$. The nonlinear structure of the space $\FFw$ complicates the proof \abbr{wrt} the over-parametrized case, and we apply standard differential geometric tools to
map the problem back to a linear space.

\begin{proof} \textbf{of \tref{t:under-parametrized}}
  We apply the rank theorem \citep{boutaib15,lee03} (\citep{abraham12} for the $\infty$-dimensional setting) to show that there exist sets
  $\WW_0, \obar \WW_0 \subseteq \Rr^p$,
  $\FF_0, \obar \FF_0 \subseteq \FF$ and diffeomorphic maps
  $\phi~:~\WW_0 \to \obar \WW_0$, $\psi~:~\FF_0 \to \obar \FF_0$ where
  $\psi \circ \hatt V \circ \phi^{-1} = \pi_r$, $\phi(w(0)) = 0$,
  $\psi(\hatt V_{w(0)}) = 0$ and, for an appropriate choice of bases,
  $\pi_r$ maps the coordinates of $\obar \WW_0$ to the \emph{first}
  $r$ coordinates of $\obar \FF_0$, \ie
  $(x_1,\dots, x_p) \mapsto (x_1, \dots, x_r, 0,0, \dots)$, where $r$
  is the rank of the operator $D V_{w(0)}$. See Figure~\ref{f:scaling} for an illustration. We denote by $\Pi_r$ the
  hyperplane in $\FF$ spanned by the first $r$ vectors of the
  basis. We recall that by \cite{boutaib15, lee03,abraham12} the maps,
  $\psi, \phi, \pi_r$ are continuous with Lipschitz derivatives
  $D \psi, D\phi, D\pi_r$ respectively.
\begin{figure}[h!]
  \centering
  \begin{tikzcd}
    \vspace{10pt}
    \WW_0 \arrow[r, "V"] \arrow[d, "\phi"]
    & \FF_0 \arrow[d, "\psi"] \\
    \obar \WW_0 \arrow[r,"\pi_r"]
    & \obar \FF_0
  \end{tikzcd}
  \caption{Illustration of the spaces and maps used in the proof of \tref{t:under-parametrized}}
\end{figure}

We consider the trajectory of $\obar V_{w(t)} := \pi_r\circ \phi(w(t)) = \psi(\hatt V_{w(t)})$. Denoting by $D \cdot$ the {Fr\'echet derivative} at the corresponding point of the dynamics and noting that $D \hatt V = D \psi^{-1} D\pi_r D \phi$ we have{
  \begin{equs}
    \frac{\d}{\d t} \obar V_{w(t)} & =  - \frac 1\alpha \dtp{D \psi D \hatt V D \hatt V^{\top},  T^\lambda \alpha \psi^{-1}(\obar V_{w(t)}) - \alpha \psi^{-1}(\obar V_{w(t)})}_\pi\\
    & = - \frac 1\alpha \dtp{D \pi_r D \phi D \phi^\top D \pi_r^\top (D
    \psi^{-1})^\top, T^\lambda\alpha \psi^{-1}(\obar
    V_{w(t)}) - \alpha \psi^{-1}(\obar V_{w(t)})}_\pi\,,\label{e:vf}
  \end{equs}}
  so $\obar V$ remains in $\Pi_r$.  As a consequence of the above we
  can naturally define a metric (the pushforward metric) on
  $\obar \FF_0$ by the tensor
  $\bar g_{\bar v} = (D \pi_r D \phi D \phi^\top D \pi_r^\top)^{-1}$. In
  fact, by choosing the metric tensor to be constant on $\obar \FF_0$, \ie
  equal to $\bar g_0$ for all $v \in \obar \FF_0$,  we equip the linear space
  $\obar \FF_0$ with a scalar product $\dtp{\,\cdot\,,\,\cdot\,}_0$.
  This, in turn, directly induces a norm $\|\,\cdot\,\|_0$ on the same
  space. We now proceed to use such simple metric structure to
  establish the existence and uniqueness of a fixed point of
  \eref{e:vf} in $\obar \FF_0$ for $\alpha$ large enough.

  The \edit{desired result} follows from 
  \cite[Proposition 4.1]{bullo14}, which establishes uniqueness and
  exponential contraction at rate $\ell > 0$ of a dynamical system
  evolving under the flow of a vector field $X$ given by the
  \abbr{rhs} of \eref{e:vf} in a forward invariant set $\obar \FF_0$
  provided that for every geodesic $\gamma(s)$ in $\obar \FF_0$
  \eref{e:convexity} holds.  Therefore, the proof of convergence
  is concluded by applying \lref{l:forwardinvariance} and
  \lref{l:convexity}.
  The proof of the optimality of the fixed point is
  postponed as \lref{l:fixedpoint}. \end{proof}

\begin{lemma}\label{l:forwardinvariance}There exists $\delta > 0$ and $\alpha_0 > 0$ such that the ball $\BB_\delta^0(0) \subseteq \obar \FF_0$
  is forward invariant and forward complete \abbr{wrt} the dynamics of \eref{e:tduavscaled} for all $\alpha > \alpha_0$.
\end{lemma}

\begin{lemma}\label{l:convexity} There exists $\ell > 0$, $\delta > 0$ and $\alpha_0 > 0$ such that for all $\alpha > \alpha_0$ and all geodesics $\gamma(s)$ contained in the ball $\BB_\delta^0(0) \subseteq \obar \FF_0$, the function
  \begin{equ}\label{e:convexity}
    {\dtp{\gamma'(s),X(\gamma(s))}_{0} - \ell s \dtp{\gamma'(0),\gamma'(0)}_{0}}\,,
  \end{equ}
  is strictly decreasing in $s$.
\end{lemma}

\begin{remark}
  The proof of \tref{t:under-parametrized} can be generalized to the case where
  $V_0$ is not identically $0$ but {within a $\|\,\cdot \,\|_\mu$-ball of radius $\rho(\alpha)$ around $0$} for $\rho(\alpha)$ going to $0$ with $\alpha \to \infty$. This generalization, however, requires the map $V$ to be \emph{uniformly} Lipschitz smooth for $w \in \WW_0$. 
  This extension allows to explicitly cover the training of randomly initialized, single layer neural networks.
\end{remark}

{The \emph{local} optimality in \tref{t:under-parametrized} is a consequence of the approximately linear behavior of lazy learners that we leverage in our proofs. Indeed, it is known \citep{hongler18} that models in this regime behave asymptotically like kernel methods, and their expressive power is therefore limited to the corresponding Reproducing Kernel Hilbert Space (\abbr{RKHS}).

In this sense, in the case of randomly initialized, wide networks it is natural to compare these models to
random feature models with an appropriate set of feature maps \citep{YehudaiShamir:19}. For a detailed comparison of random feature models and neural networks in the lazy training regime we refer the reader to \cite{Montanari:19c}. Furthermore, we refer to \citet{bach17} for a general discussion of the approximating power of neural network models in this lazy regime, highlighting their limits \abbr{wrt} their non-lazy counterparts.}

\subsection{Mean-field regime} \label{s:mf}

To better understand the limitations of training wide neural networks in the lazy training regime, we contrast their behavior with the one of networks that do not display such approximately linear behavior through training.
These models, introduced below, correspond to a different scaling of the model parameters at initialization and are commonly referred to as mean-field models.
While results on the convergence of such mean-field models are obtained in a more restricted setting detailed below, they provide an effective term of comparison to the linearized regime explored in previous sections.

{In the following we set, for clarity of exposition, $\lambda = 0$ and we assume that $|\mathcal S| = \infty$ in order to be in the under-parametrized regime of interest.
We further restrict to the setting where the function $\psi$ is $h$-homogeneous for $h\geq 1$ in at least one of its parameters. A function $f$ is called $h$-homogeneous if $f(\xi x) = \xi^h f(x)$ for all $\xi \in \Rr$. One simple case where this holds is when
  {$\omega = (\omega_0, \bar \omega) \in \Rr \times \Theta$ for $\Theta \subseteq \Rr^{m-1}$ and $\psi(s; \omega) = \omega_0 \phi(x;\bar \omega)$} for a certain, usually nonlinear function $\phi~:~\SS \times\Theta \to \FF$, so that $\psi$ is $1$-homogeneous in $\omega_0\in \Rr$. This is for instance the case in the setting of single (hidden) layer neural networks, radial functions, and linear models,
where $\phi$ represents, respectively, the nonlinear activation function, the radial basis function or the model features.}

\paragraph{The mean-field regime} By the assumed structure of the approximator, we represent $V_w$ in \eref{e:meanfield} as
$V_{w}(s) = V_{\nu^{(N)}}(s) := \int_{\Omega} \psi(s;\omega ) \nu^{(N)}(\d \omega)$ {where }$\nu^{(N)}(\d \omega ) = \frac 1 N \sum_{i=1}^N \delta_{\wii}(\d \omega) \in \mathcal M_+(\Omega)$.
This empirical measure representation removes the symmetry of the approximating functions under permutations of the weight indices, while behaving naturally in the limit $N \to \infty$, when $\nu^{(N)} \to \nu $ weakly, so that $V_{\nu^{(N)}} \to V_{\nu}$.
\edit{Upon rescaling time as $t \leftarrow N t$, } the evolution  {of the measure $\nu\in \mathcal M_+(\Omega)$ }
\edit{under \eref{e:tduc}} is then governed by a \emph{mean-field} transport partial differential equation of the Vlasov type
, given by
\begin{equ}\label{e:mfpde}
  \frac {\d}{\d t} \nu_t(\omega) = \mathrm{div}\pc{\nu_t(\omega) \int_{\SS\times \SS} \nabla_\omega \psi(s;\omega) \delta(s,s',\nu_t) P(s,\d s')\mu(\d s)}
\end{equ}
for the \abbr{TD}-error $\delta$ from \eref{e:tde1}, \ie $
  \delta(s,s',\nu) := {r(s,s') + \gamma \int\psi(s';\omega')\nu(\d \omega') - \int\psi(s;\omega')\nu(\d \omega')}$. {Analogous dynamics equation in the supervised learning case has been derived
  in \cite{MeiMonNgu18, RotVE18, ChizatBach18}.}

 {To state the main result of this section, the optimality of fixed points of \eref{e:mfpde} we need the following}
\begin{assumption} \label{a:mf} Assume that $\omega = (\omega_0, \bar \omega) \in \Rr \times \Theta = \Omega$ for $\Theta =  \Rr^{m-1}$ and $\psi(s; \omega) = \omega_0 \phi(s;\bar \omega)$ with
  \begin{enumerate}[label=\alph*)]
    \item {\rm Regularity of $\phi$:} 
    $\phi$ is bounded, differentiable and $D\phi$ is Lipschitz.  {Also, for all $f \in \FF$ the regular values of the map {$\bar \omega \mapsto g_f(\bar \omega):= \dtp{f,\phi(\,\cdot\,; \bar \omega)}$} are dense in its range, and 
     $g_f(r\bar \omega) $ converges in $C^1(\{\bar \omega \in \Theta:\|\bar \omega\|_2=1\})$ as $r \to \infty$ to a map $\bar g_f(\bar \omega)$ whose regular values are dense in its range.}
    \item {\rm Universal approximation}: the span of $\{\phi(\cdot, \bar \omega)\,:\,\bar \omega \in \Theta \}$ is dense in $L^2(\mathcal S,\mu)$
    \item {\rm Support of the measure}: {There exists $r_0 > 0$ s.t. {the 
        support of the initial condition $\nu_0$} is contained in} $\mathcal Q_{r_0} = [-r_0,r_0]\times \Theta$ and separates $\{-r_0\}\times \Theta$ from $\{r_0\}\times \Theta$, {\ie if any continuous path connecting $\{-r_0\}\times \Theta$ to $\{r_0\}\times \Theta$ intersects the support of $\nu_0$.}
  \end{enumerate}
\end{assumption}

\aref{a:mf} a) is a common,  {technical} regularity assumption  {(\eg \cite[Assumption 3.4]{ChizatBach18}),} ensuring that \eref{e:mfpde} is well behaved and controlling the growth, variation and regularity of $\phi$. Alternative assumptions on the case $\Theta \neq \Rr^{m-1}$ are given in the appendix.
 \aref{a:mf} b) speaks to the approximating power of the nonlinearity, assumed to be powerful enough to approximate any function in $L^2(\mathcal S,\mu)$,  {while c) guarantees that the initial condition is such that the expressivity from b) can actually be exploited. \edit{Universal approximation assumptions similar to \aref{a:mf} b) were made in \cite{lu20,nguyen20}}.
 }

Before discussing the optimality properties of the dynamics \eref{e:mfpde}, we show that this \abbr{pde} accurately describes the \abbr{td} dynamics of a sufficiently wide, single layer neural network. To this aim, we let $\mathcal P_2(\Omega)$ be the space of probability distributions on $\Omega$ with finite second moment.

 \begin{proposition}\label{t:particles}
   Let \aref{a:mf} hold and let $w_t^{(N)}$ be a solution of \eref{e:tduc} with initial condition $w_0^{(N)} \in \mathcal W = \Omega^N$. If $\nu_0^{(N)}$ converges to $\nu_0 \in \mathcal P_2(\Omega)$ in Wasserstein distance $W_2$ as $N \to \infty$ then $\nu_t^{(N)}$ converges, for any $t>0$, to the unique solution $\nu_t$ of \eref{e:mfpde}.
 \end{proposition}
 We note that by the law of large numbers for empirical distributions, the condition of convergence of $\nu_0^{(N)}$ to $\nu_0$ is \eg satisfied when $w_0^{(i)}$ are drawn independently at random from $\nu_0$.
The proof of the above result is largely standard under the given assumptions, and was given in a setting similar to the one at hand in \cite{wang20td}. For completeness, we provide a sketch of the proof in the appendix. The idea of the proof is a canonical propagation of chaos argument \citep{sznitman91}, which we adapt from \cite{ChizatBach18} to the present context.
In the absence of an energy functional in the \abbr{td} setting, in order to establish the necessary regularity of the vector field we prove
\begin{lemma}\label{l:bounded}
  For any $\nu_0$ with $\int \omega_0^2 \nu_0(\d \omega) \leq \infty$ there exists $C_V>0$ such that, for any $t>0$ we have $\|V_{\nu_t}\|_\mu < C_V$.
\end{lemma}
To prove the above result we leverage the homogeneity of the approximator, slightly adapting the proof of \cite[Theorem 1]{BrandfonbrenerBruna} to the present setting. We note that the above result is of independent interest in that it rules out divergence in predictor space of the mean-field dynamics. We are now ready to state the main optimality result of this section:
\edit{\begin{theorem}[Mean-field optimality]\label{t:mf}
  Let \aref{a:mf} hold and $\nu_t$ given by \eqref{e:mfpde} converge
  to $\nu^*$, then $V_{\nu^*} = V^*$ $\mu$-a.e.
\end{theorem}}
Thus if the \abbr{TD}-learning dynamics \eref{e:mfpde} converges to a fixed point,
{it is a global minimizer.} \edit{To prove this result, we first connect the optimality of a fixed point with the support of the underlying measure in parameter space.
More specifically, we show in \lref{p:mf} that by the expressivity of $\sigma$, the transport vector field of suboptimal fixed points of the dynamics \eref{e:mfpde} cannot vanish everywhere in parameter space, so that a measure with sufficient support cannot correspond to a suboptimal fixed point.
}

\edit{We then show in \lref{l:support} that the \abbr{TD} dynamics in the mean-field regime preserves \edit{such sufficient notion of support of the measure (\aref{a:mf} c))} throughout training. This is true for any finite time by topological arguments: the separation property of the measure cannot be altered by the action of a continuous flow such as \eref{e:mfpde}. }
      {Leveraging the above partial results, we finally prove in \lref{l:spurious}, that spurious fixed points are avoided by the \abbr{td} dynamics \eref{e:mfpde} when initialized properly. To establish this we argue by contradiction: assuming that we are approaching such a spurious fixed point $\tilde \nu$ at time $t_0$, we show in \lref{l:boundong} that the velocity field will change little for any $t>t_0$. 
      On the other hand, by the homogeneity of $\psi$ and by \aref{a:mf} c), we show that by \lref{l:support} a positve amount of measure $\tilde \nu$ will fall in a forward invariant region \edit{-- which exists by \lref{p:mf} --} where its $\omega_0$ component will grow linearly in $t$, thereby eventually contradicting the assumption that $\tilde \nu$ is a fixed point of \eref{e:mfpde}.

\section{Numerical examples}\label{s:numerical}

\subsection{A divergent nonlinear approximator}\label{s:tsitsiklis}

We illustrate the convergence properties of \abbr{TD} learning in the lazy training regime in the under-parametrized case by applying it to the classical framework of \citet[Section X]{tsitsiklis}. This reference gives an example of a family of nonlinear function approximators that diverge when trained with the \abbr{TD} method. The intuition behind this counterexample is that one can construct a manifold of approximating functions $\FFw$ in the form of a spiral, with the same orientation as the rotation of the vector field induced by the \abbr{TD} update in the space of functions. By choosing the windings of the spiral to be dense enough, the projection of the \abbr{TD} vector field follows the spiral in the outward direction, leading to a divergence of the algorithm, as displayed schematically in \fref{f:spiral1}. More specifically, consistently with \citet{tsitsiklis}, we parametrize the manifold $\FFw$ as $V_\theta := e^{\hat \epsilon \theta} (a \cos( \hat \lambda \theta ) - b \sin (\hat \lambda \theta)) - V^*$ for $a = (10,-7,-3)$, $b = (2.3094,-9.815,7.5056)$, $\hat \epsilon = 0.01$, $\hat \lambda = 0.866$. We set $\gamma = 0.9$ and step-size
$\beta_t \equiv 2\times 10^{-3}$, while the underlying Markov chain is defined by the transition matrix $P_{ij} = (\delta_{j,\mathrm{mod}(i,3)+1}+\delta_{i,j})/2$, where $\delta_{i,j}$ is the Kronecker delta function and equals $1$ if $i = j$ and $0$ else. We note that the step-size does not affect the convergence properties of the algorithm, as argued in \citet{tsitsiklis}, where the immediate reward was set to $\bar r = (0,0,0)$.
Note that, as realizing the conditions of \tref{t:under-parametrized} starts the simulation at the solution $V^* = (0,0,0)$, we shift both the solution and the manifold of approximating functions by the same vector in the embedding space, leaving the new solution $V^*= - V_0 = -a$ at the center of the spiral, \ie realized at $\theta = -\infty$. This corresponds to choosing an average reward $\bar r = (-6.85, 8.35, -1.5)$. We note that by the affine nature of the \abbr{TD} update, this change in $\bar r$ results in a global shift of the \abbr{TD} vector field in $\FF$ and does not affect the update of $\theta$. In particular, this means that the \abbr{TD} update remains \emph{divergent} for every initial condition different than the solution $V^*$.
 \begin{figure}
\centering
\subfigure[$\alpha = 1$]{
  \includegraphics[width=.30\linewidth]{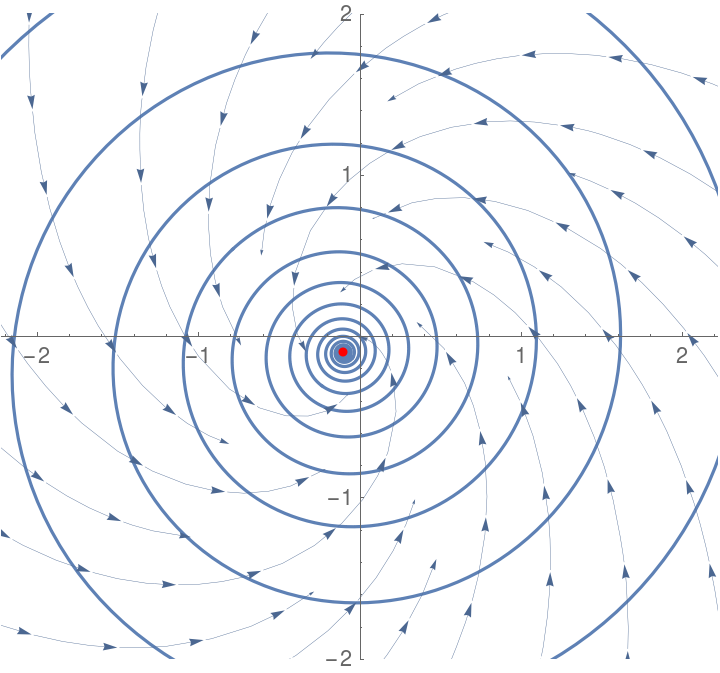}
  \label{f:spiral1}
}
\subfigure[$\alpha = 100$]{
  \includegraphics[width=.30\linewidth]{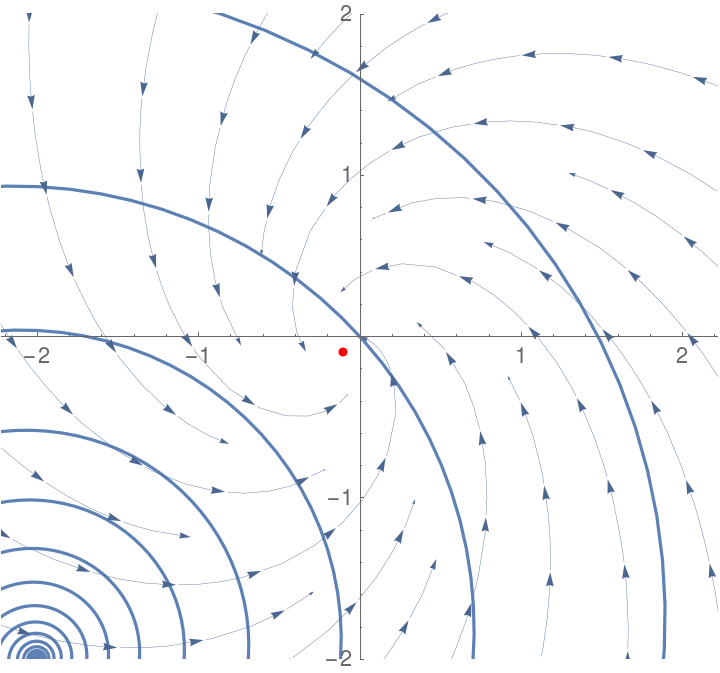}
  \label{f:spiral2}
}
\subfigure[Training dynamics]{
  \includegraphics[width=0.35\linewidth]{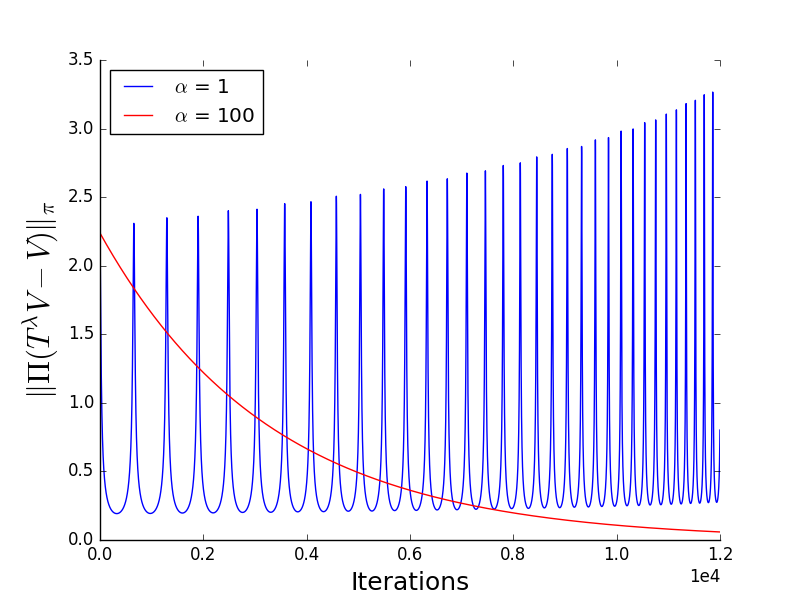}
  \label{f:result1}
}
 \caption{Schematic representation of the manifold $\FFw$ for the
  example in \sref{s:tsitsiklis} before (a) and after (b) scaling of
  $\alpha$. The underlying vector field represents the \abbr{TD} error
  $\delta(V)$ from \eref{e:tde1}, whose projection on $\edit{{\mathcal T}_{V}} \FFw$
  gives the dynamics of the \abbr{TD} update in $\FFw$.  In (a) this
  projection points ``outwards'' along the spiral, while (b) it has a
  fixed point close to $0$. The scaling yields an effective
  ``linearization'' of the manifold around $0$. The red point
  marks the global fixed point of the vector field. In (c), we plot the $\mu$-norm of the projected \abbr{TD} error $\Pi(T^\lambda V - V)$. This quantity measures the increments of the model parameters during training and vanishes at a local minimum of the \abbr{TD} dynamics. We see that the algorithm diverges for $\alpha = 1$ (blue curve), but converges to a local minimum for $\alpha = 100$ (red curve).
  }
\label{f:spiral}
\end{figure}
We run the \abbr{TD} update in the off-centered situation both for values of $\alpha = 1$ (the classical, divergent regime) and $\alpha = 100$. As explained in the previous sections, this scaling of the approximating function makes the \abbr{TD} update \emph{convergent}, as displayed in \fref{f:result1}. The intuition behind the convergence of the algorithm is outlined in \fref{f:spiral}: when $\alpha$ is large we are in an almost linear regime where the \abbr{TD} update converges to a \emph{local} minimum of the dynamics.

\subsection{Single layer neural networks} \label{s:nn}

\begin{figure}[t]
\centering
\subfigure[Lazy training]{
  \includegraphics[width=0.43\linewidth]{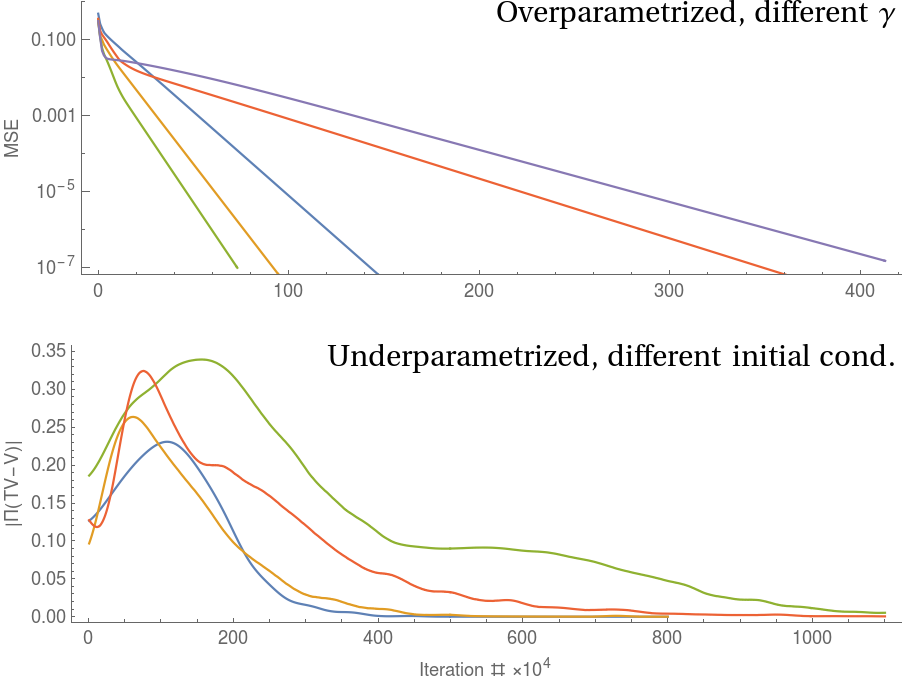}
  \label{f:result2}
}
\subfigure[Mean-field]{
  \includegraphics[width=0.43\linewidth]{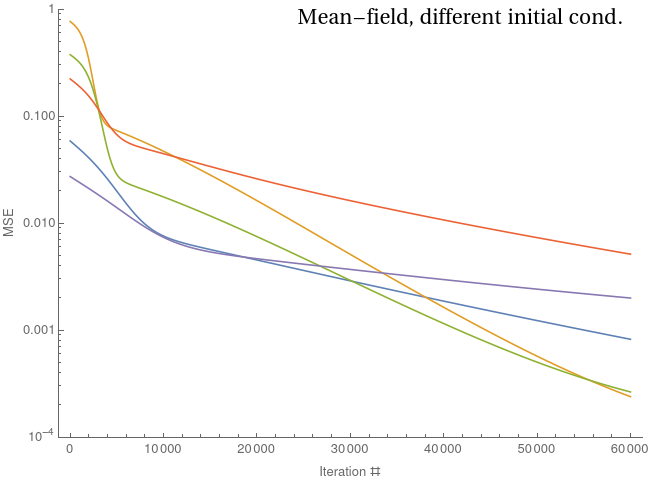}
  \label{f:fig2}
}
\caption{Training results 
or the examples described \sref{s:nn} in the lazy (a) and mean-field (b) regimes. In (a, top) we plot the \abbr{mse} of single layer neural network during training in the over-parametrized regime ($N = 100, d = 30, \alpha = 500$) for different choices of $\gamma$
($0.8,0.83,0.85,0.87,0.9$), showing exponential convergence (at different rates) to the global minimum claimed in \tref{t:over-parametrized}.
 In (a, bottom) we plot the norm of the projected \abbr{TD} error for a neural network in the under-parametrized regime ($N = 10, d = 50, \alpha = 100, \gamma = 0.9$) for different initial conditions, showing convergence to a local fixed point. \aa{In (b) we again plot the \abbr{MSE} in the mean-field regime ($N = 350, d = 700, \alpha = 1, \gamma = 0.9$) for different initial conditions. This plot displays qualitatively different convergence behaviors, \edit{with} no apparent upper bound on the convergence rate, indicative of the nonlinear character of the model. Code available here: \cite{code}. }
 }
\label{f:results}
\end{figure}

We show that the regime of study arises naturally in one hidden layer neural networks for a certain family of initialization.
We consider the example of ReLu activation, \ie when the model is given by
$V_w(s) =  \sum_{i = 1}^N a_i \max(0, b_i \cdot s - c_i)\,,$
for $s \in \Rr^m$ and $N$ distinct $(m+2)$-dimensional vectors $w_i = (a_i,(b_i)_1,\dots, (b_i)_m,c_i)_{i \in (1, \dots, N)}$. Typical initialization of the weights of the above model is of the form $ a_i \overset{iid}\sim \NN(0,1/\sqrt{N})$, $ (b_i)_j \overset{iid}\sim \NN(0,1/\sqrt{m})$ for all $j$ and $ c_i \overset{iid}\sim \NN(0,1)$. However, by the linearity of the activation function in $a_i$ and by the rescaling property of normal distribution this is equivalent to writing $\alpha V_w(s) = \alpha \frac{1}{N} \sum_{i = 1}^N a_i \max(0, b_i \cdot s - c_i)$
for an $N$-dependent $\alpha(N) = \sqrt{N}$ (diverging in $N$), $a_i \overset{iid}\sim \NN(0,1)$, $ (b_i)_j \overset{iid}\sim \NN(0,1/\sqrt{d})$ and $ c_i \overset{iid} \sim \NN(0,1)$. 
 Therefore, this common choice of initial conditions  implicitly starts the training of the above model in the lazy regime \citep{Montanari:19c}. We train the network by \abbr{TD} learning \eref{e:tduavscaled} with fixed step-size $\beta_t \equiv 10^{-3}$ both in the over- and under-parametrized regime. To do so, we draw an objective function $V^*$ randomly with distribution $ V^*(s) \overset{iid}{\sim} \NN(0,1)$ for all $s \in \SS$ on a grid of $d$ equally spaced points on the interval
 $[-1,1]$.
We then compute the corresponding average reward by solving
$ \bar r =(\indicator -\gamma P )V^*$, and train the model \eref{e:tduavscaled} for $\lambda = 0$, $\gamma = 0.9$ (when not specified otherwise) with transition matrix $P_{ij} = (\delta_{j,\mathrm{mod}(i,d)+1}+\delta_{i,j})/2$.  To respect the conditions of \tref{t:under-parametrized}, we initialize half  {of the model parameters} as explained above, while the other half is obtained from the first by replicating the values of $b_i, c_i$ and inverting the one of $a_i \to - a_i$. This ``doubling trick'' introduced in \citet{cb182} produces a neural network with $V_{w(0)} \equiv 0$ and randomly initialized weights with the desired distribution. We consider situations where $N = 10$, $d = 50$ (under-parametrized, taking $\alpha = 100$) and $N = 100$, $d = 30$  (over-parametrized, with $\alpha = 500$), and plot the convergence to local, respectively global minima in \fref{f:result2}.

To illustrate the convergence properties of neural networks in the mean-field regime as presented in \tref{t:mf} we repeat the above experiment taking $\alpha = 1$. In this setting, we consider a model with parameters $N = 350$, $d = 700$, resulting in the under-parametrized regime. In this case, to ensure that the value function to be learned can be approximated by the finite-width network we consider a target value function $V^*(s)$ given by a neural network of the form \eref{e:meanfield} with $N^* = 4$, initialized as $a_i^* \overset{iid}\sim \NN(0,1)$, $ (b_i^*)_j \overset{iid}\sim \NN(0,1/\sqrt{d})$ and $ c_i^* \overset{iid} \sim \NN(0,1)$. The results of these simulations are displayed in \fref{f:fig2}. We notice that, while the MSE converges to zero, the rate of decay of this quantity does not display the regularity observed in the over-parametrized, lazy regime, numerically corroborating the hypothesis that the convergence of mean-field models is not \emph{uniformly} linear.

%

\section{Conclusions and future work}\label{s:conclusion}

In this work we have discussed the convergence and optimality properties of the
\abbr{td} learning algorithm with wide neural networks as function approximators, comparing the effect of different initialization procedures on the dynamics of the algorithm and on the limiting approximator, both in the under- and over-parametrized setting.

In the lazy regime, the
algorithm behaves essentially like a linear approximator spanning the
tangential space of the approximating manifold {(in function space)} at initialization. As
such, the training converges exponentially fast with probability one
to the global minimum or a local fixed point depending on the
codimension of the approximating manifold in the search space.
This phenomenon can be understood as an effect of the
linearized regime in which the neural networks are trained which
reduces them, in the limit, to a kernel method \citep{hongler18}. This somewhat limited expressivity is reflected in the \emph{local} optimality of the fixed points of the \abbr{td} dynamics in the lazy regime.

We contrast this behavior by proving \emph{global} optimality guarantees for the fixed points of the same models when trained in the mean-field scaling limit, the highly nonlinear dynamics of which make convergence (and convergence rate bounds) hard to establish.
In this sense, we argue that convergence of lazy models comes at the expense of their expressivity.
  {Nonetheless, the results proven in this work emphasize the interest of the lazy regime in the framework of deep reinforcement learning, where models often suffer from divergent behavior especially during early stages of training. Optimal results are expected given prior knowledge that the value function belongs to a known \abbr{RKHS}: using a nonlinearity whose \abbr{ntk} includes such \abbr{rkhs} would result in both convergence and optimality guarantees of the learned value function.}



\edit{Future directions of research include the development of more refined, nonasymptotic versions of the above theorems and the extension of these results to more complex, nonlinear reinforcement learning algorithms in the Markov Decision Processes setting.
  We note that such extension is immediate in settings such as SARSA, where the state-action space can be regarded as an extension of the state space in the \abbr{TD} case. On the contrary, the generalization of the results obtained above to off-policy algorithms such as Q-learning pose a more significant challenge. This is captured for instance in \cite{meyn08}, where convergence guarantees for linear approximators can only be obtained upon making relatively strong a-priori assumptions on the spectral properties of the feature matrix during training. Consequently this direction of research remains open.} {Furthermore, a more thorough exploration of the relationship between the limiting results in \citet{ChizatBach18} and the ones presented here and in \citet{cb182} would be important for the understanding of the limiting dynamics of neural networks in this domain, in particular in terms of convergence of the nonlinear, mean-field limit in the context of reinforcement learning.} 
%
%

\acks{We thank a bunch of people.}

\bibliographystyle{JPE}
\bibliography{bib.bib}

\appendix

\renewcommand\thefigure{\thesection.\arabic{figure}}

\section{Supplementary proofs in the lazy training regime}

{To simplify the notation in the forthcoming analysis, we slightly abuse the notation used when the state space is finite-dimensional extending it, when necessary, to the infinite-dimensional setting. This naturally generalizes matrix multiplication to the action of linear operators. In particular the action of $\Gamma$, which we recall in the finite-dimensional setting is a diagonal matrix with entries $\mu(s)$, is to be intended as
\begin{equ}
  (a^\top \Gamma b)_{ij} = \int_\SS a_i(s) b_j(s) \,\mu(\d s)\,.
\end{equ}}
  Furthermore, we introduce the following decomposition of the \abbr{td} operator:
\begin{equ}
 T^\lambda V = \bar r^\lambda + \gamma P^\lambda V\,,
\end{equ}
where
\begin{equ}
 \bar r^\lambda(s) := (1-\lambda) \sum_{m = 0}^\infty \lambda^m \Ex{s}{\sum_{k = 0}^m \gamma^k r(s_k,s_{k+1})}\,,\quad P^\lambda V(s) := (1-\lambda) \sum_{m = 0}^\infty (\lambda\gamma)^{m} \Ex{s}{ V(s_{m+1})}\,,
\end{equ}
or, in vector notation
\begin{equ}
 \bar r^\lambda := (1-\lambda) \sum_{m = 0}^\infty \lambda^m {\sum_{k = 0}^m \gamma^k P^k r}\,,\quad P^\lambda V(s) := (1-\lambda) \sum_{m = 0}^\infty (\lambda\gamma)^{m}  P^{m+1}V\,.
\end{equ}
In the proofs below, we will use the above, simplified notation to obtain contraction estimates on the dynamical system \eref{e:tduc}. These estimates will leverage the fact that $P^\lambda$ is nonexpansive and $\gamma < 1$, and from this notation contraction rates in terms of $\gamma$ will arise naturally. However, by \lref{l:tsitsiklis}, we know that the contraction rate of $T^\lambda$ is $\gamma_\lambda$. Rewriting the proofs with $\gamma \to \gamma_\lambda$ will show the stronger contraction.

\subsection{Over-parametrized regime}

\begin{replemma}{l:perturbation}
    Let $\GG_0$ be a compact subset of a linear space $\GG$. For {$v(0) \in \GG_0$}, let $g_v$ be a continuous, {self-adjoint linear operator} that is positive definite in a neighborhood of $v(0)$ when restricted on $\GG$. Then for all $\epsilon > 0$ there exists $\delta > 0$ such that, for all $v \in \BB_\delta(v(0)) \subseteq \GG_0$
  \begin{equ}
    g_{v(0)} =  (\indicator + \tilde  g_v)g_v\,,
  \end{equ} for a linear operator $\tilde g_v~:~\FF\to\FF$ with
  $  \|\tilde g_v\|< \epsilon\,.$
  More specifically, let $\sigma_{\text{min}}$ be the smallest singular value of $DV_{w(0)}$. Then if $\rho \leq ({1-\gamma})\sigma_{\text{min}}^2/(48 L_{D \hatt V})$,
   \eref{e:perturbation} holds
  with $\|\tilde g_{V(w)}\|< \frac{1-\gamma}4$ for all $w \in \BB_{\rho}(w(0))$.
\end{replemma}

\begin{proof}\textbf{of \lref{l:perturbation}}
 Let $B_w =  D \hatt V_{w(0)} D\hatt V_{w(0)}^\top - D \hatt V_{w} D\hatt V_{w}^\top$. We carry out the proof for the case $\sigma_{\text{min}}< 1$ (else the result holds with $\sigma_{\text{min}} = 1$ in $\rho$), in which case we have for all $w \in \BB_\rho(w(0))$ that
 \begin{equ}
  \|B_w\| \leq 2 L_{D\hatt V} \|w(0)-w\| + (L_{D\hatt V} \|w(0)-w\|)^2 \leq 3 L_{D \hatt V}\|w(0)-w\|\,.
 \end{equ}
 Then we can write
 \begin{equs}
  g_{w(0)} & = (D \hatt V_{w(0)} D\hatt V_{w(0)}^\top)^{-1} = {(D \hatt V_{w} D\hatt V_{w}^\top + B_w)^{-1} }\\
  & = (g_w^{-1} (1+g_wB_w  ) )^{-1} = (1+ g_w B_w )^{-1} g_w \\
  & =  \sum_{n = 0}^\infty (-1)^n (g_w B_w )^n g_w = g_w + \sum_{n = 1}^\infty (-1)^n  (g_wB_w )^n g_w\,.
 \end{equs}
 Furthermore, by the assumptions on the regularity of $V$ and on the initial condition $w(0)$ we have that $g_w \preceq 4/\sigma_{\text{min}}^2 \indicator $, provided that $w \in \mathcal B_{\rho}(w(0))$ for $\rho$ as in \lref{l:perturbation}. Therefore, the perturbation $\tilde g_w := \sum_{n = 1}^\infty (-1)^n  (g_wB_w )^n$ satisfies
 \begin{equ}
  \|\tilde g_w\| =  \|\sum_{n = 1}^\infty (-1)^n  (g_w B_w)^n\| \leq \sum_{n = 1}^\infty  \|g_w B_w \|^n  \leq \sum_{n = 1}^\infty \pc{  \frac{3 L_{D\hatt V}}{\sigma_{\text{\scriptsize min}}^2/4} \|w(0)-w\|}^n \leq \frac{1-\gamma}4\,.
 \end{equ}
 The same proof applies in the general case with different, implicit constants.
\end{proof}

\begin{replemma}{l:tsitsiklis}\cite[Lemmas 1, 3, 7]{tsitsiklis}
Under \aref{a:P}, for any $V, \tilde V \in \FF$ we have that
\begin{equ}\label{e:TDcontractionsupp}
  \|T^{\lambda} V- T^{\lambda} \tilde V\|_\mu \leq \gamma_\lambda \| V- \tilde V\|_\mu \qquad \text{for }\quad \gamma_\lambda := \gamma\frac{ 1-\lambda}{1-\gamma \lambda} \leq \gamma< 1\,.
\end{equ}
In particular there exists a unique fixed point of $T^\lambda$, $V^* \in \FF$ given by \eref{e:extremum}.
\end{replemma}

\begin{proof}\textbf{of \lref{l:tsitsiklis}}
 We first prove that $\|P V\|_\mu \leq \|V\|_\mu$. This follows by Jensen inequality and by the invariance of $\mu$:
 \begin{equs}
  \|P V\|_\mu^2 &= V^\top P^\top \Gamma P V = \int_\SS \mu(\d s) (\int_\SS P(s,\d s') V(s'))^2 \\
  &\leq \int_{\SS^2} \mu(\d s)  P(s,\d s') V(s')^2 =\int_\SS \mu(\d s) V(s)^2 = \|V\|_\mu^2\,.\label{e:pcontraction}
 \end{equs}
 Then, writing
 \begin{equs}
  T^\lambda V(s) & = (1-\lambda) \sum_{m = 0}^\infty \lambda^m \Ex{s}{\sum_{k = 0}^m \gamma^k r(s_k,s_{k+1}) + \gamma^{m+1} V(s_{m+1}) } \\&= (1-\lambda) \sum_{m = 0}^\infty \lambda^m \pc{ \sum_{k = 0}^m \gamma^k \Ex{s}{ \bar r(s_k)} + \gamma^{m+1} \Ex{s}{V(s_{m+1}) }}\\
  &= (1-\lambda) \sum_{m = 0}^\infty \lambda^m \pc{ \sum_{k = 0}^m \gamma^k P^k \bar r(s) + (\gamma P)^{m+1} {V(s) }}\,,
 \end{equs}
 where $s_k$ is the process on $\SS$ induced by $P$ with initial condition $s_0$, we have contraction of the operator $T^\lambda$ in $L^2(\SS,\mu)$ by
 \begin{equs}
  \|T^\lambda (V-\tilde V)\|_\mu & = \left\|(1-\lambda) \sum_{m = 0}^\infty \lambda^m  (\gamma P)^{m+1} \pc{V(s) - \tilde V(s) } \right\|_\mu \\& \leq (1-\lambda) \sum_{m = 0}^\infty \lambda^m  \gamma^{m+1} \left\| {V(s) - \tilde V(s) } \right\|_\mu \\
  &= \frac{\gamma (1-\lambda)}{1-\gamma \lambda} \left\| {V(s) - \tilde V(s) } \right\|_\mu\,,
 \end{equs}
 where in the inequality above we have used \eref{e:pcontraction}. This proves that $T^\lambda$ is a contraction in $\FF$, and as such it must have a unique fixed point. That this fixed point corresponds to \eref{e:extremum} is immediately checked by direct computation.
\end{proof}


\begin{reptheorem}{t:over-parametrized}
  Assume that $\sigma_{\text{min}}>0$, where $\sigma_{\text{min}}$ is
  the smallest singular value of $D\hatt V_{w(0)}$. Assume further that $w(0)$ is such that
  $\|\hatt V_{w(0)}\|_0 < M := ({1-\gamma})^2\sigma_{\text{min}}^2/(192
  \kappa^2 L_{D \hatt V}\|D\hatt V_{w(0)}\|)$, then for
  $\alpha > \alpha_0 := \|V^*\|_0/M$ we have for all $t \geq 0$ that
\begin{equ}\label{e:expcontractionsupp}
  \|V^*- \alpha \hatt V_{w(t)}\|_0^2 \leq \|V^*- \alpha \hatt V_{w(0)}\|_0^2 e^{- \frac{1-\gamma}{2\kappa^2 } t }\,.
\end{equ}
Recall that $V^*$ is the exact value function given by \eref{e:extremum}. Moreover, if $\|\hatt V_{w(0)}\|_0 \leq C \alpha^{-1}$ for a constant $C > 0$, then $\sup_{t>0}\|w(t)- w(0)\| = \mathcal O(\alpha^{-1})$.
\end{reptheorem}

\begin{proof}\textbf{of \tref{t:over-parametrized}}
 By setting
 $\rho := {({1-\gamma}) \sigma_{\text{min}}^2/(48 L_{D \hatt V})}$ and
 by the assumed Lipschitz smoothness of $\hatt V$,
 $D\hatt V_w \cdot D\hatt V_w^\top \succeq \sigma_{\text{min}}^2/4$ as long
 as $w \in \BB_{\rho}(w(0))$. We would like to check a local exponential
 contraction condition, \ie that for all $w(t) \in \BB_\rho(w(0))$ we have
 \begin{equ}\label{e:lyapunovcond}
  \frac{\d}{\d\,t} U(\alpha V_{w(t)}) \leq \frac{\gamma-1}{2\kappa^2 } U(\alpha V_{w(t)})\,, \qquad \text{for } t > 0\,.
 \end{equ}
 To obtain the above result we apply the chain rule:
 \begin{equs}
  \frac{\d}{\d\,t} U(\alpha V_{w(t)}) &=  \dtp{\partial_f U(\alpha V_{w(t)})\,,\, \frac{\d}{\d\,t} \alpha V_{w(t)}}_0\\
  & = \alpha \dtp{\alpha V_{w(t)} - V^*\,,\, D V_{w(t)} \cdot \frac{\d}{\d\,t} w(t)}_0\\
  & = \dtp{\alpha V_{w(t)} - V^*\,,\, D V_{w(t)} \cdot D V_{w(t)}^\top \Gamma (T^\lambda \alpha V_{w(t)}- \alpha V_{w(t)})}_0\,.\label{e:firstpart}
 \end{equs}
 Throughout, we define $\tau_\rho := \inf\{t < 0 ~:~ w(t) \not \in \mathcal B_{\rho}(w(0))\}$, $g_w:= (D V_{w} \cdot D V_{w}^\top)^{-1}$ (recalling that the $D V_{w} \cdot D V_{w}^\top$ has full rank in $\BB_\rho(w(0))$) and write
 $g_0 = (\indicator + \tilde g_w) g_w$, where $\tilde g_w$ is defined in \lref{l:perturbation}.
 Then, as long as $t < \tau_\rho$ we have, for every $a,b \in \FF$
 \begin{equ}
  \dtp{a, D V_{w(t)}\cdot D V_{w(t)}^\top \Gamma b}_0 = \dtp{a, (\indicator + \tilde g_{w(t)}) \Gamma b} \leq \dtp{a, b}_\mu + \| \tilde g_{w(t)}\|\|a\|_\mu\|b\|_\mu\,.
 \end{equ}
 By the above result we can bound from above the \abbr{rhs} of \eref{e:firstpart} by
 \begin{equ}
  \frac{\d}{\d\,t} U(\alpha V_{w(t)}) \leq \dtp{\alpha V_{w(t)} - V^*, T^\lambda \alpha V_{w(t)}-\alpha V_{w(t)}}_\mu + \| \tilde g_{w(t)}\|\|\alpha V_{w(t)} - V^*\|_\mu\|T^\lambda \alpha V_{w(t)}-\alpha V_{w(t)}\|_\mu\,. \label{e:sum}
 \end{equ}
 Recalling that by \lref{l:tsitsiklis} we have
 \begin{equ}\label{e:TDpart}
  \|T^\lambda \alpha V_{w(t)}-\alpha V_{w(t)}\|_\mu = \|T^\lambda \alpha V_{w(t)}-V^*\|_\mu +  \|\alpha V_{w(t)} - V^*\|_\mu \leq 2 \|\alpha V_{w(t)} - V^*\|_\mu\,,
 \end{equ}
 and applying \lref{l:perturbation}, we can bound the second term of \eref{e:sum} from above as
 \begin{equ}\label{e:part1}
  \| \tilde g_{w(t)}\|\|\alpha V_{w(t)} - V^*\|_\mu\|T^\lambda \alpha V_{w(t)}-\alpha V_{w(t)}\|_\mu \leq  \frac{1- \gamma}2 \|\alpha V_{w(t)} - V^*\|_\mu^2\,.
 \end{equ}
 On the other hand, for the first term we have by Cauchy-Schwartz inequality and \eref{e:TDcontractionsupp} that
 \begin{equs}
  \dtp{\alpha V_{w(t)} - V^*, T^\lambda \alpha V_{w(t)}-\alpha V_{w(t)}}_\mu  & = \dtp{\alpha V_{w(t)} - V^*\,,\,  (T^\lambda \alpha V_{w(t)}- V^*)-(\alpha V_{w(t)}-V^*)}_\mu\,,\\
  & \leq    \|\alpha V_{w(t)}-V^*\|_\mu \| T^\lambda \alpha V_{w(t)}-V^*\|_\mu -  \|\alpha V_{w(t)}-V^*\|_\mu^2 \\
  & \leq({\gamma-1}) \|\alpha V_{w(t)}-V^*\|_\mu^2\,,\label{e:part2}
 \end{equs}
 where $\gamma$ is the contraction rate of the \abbr{TD} difference in
 $\FF$, see \eqref{e:TDcontractionsupp}.  Finally, combining
 \eref{e:part1} and \eref{e:part2} we obtain
 \begin{equ}
  \frac{\d}{\d\,t} U(\alpha V_{w(t)}) \leq  \frac{\gamma-1}{2}  \|\alpha V_{w(t)}-V^*\|_\mu^2 \leq  \frac{\gamma-1}{2\kappa^2 }  \|\alpha V_{w(t)}-V^*\|_0^2\label{e:final1}\,,
 \end{equ}
 and the last inequality results from the equivalence of norms $\|\,\cdot\,\|_0$ and $\|\,\cdot\,\|_\mu$ (both have full support on a finite set).  The desired result \eref{e:expcontractionsupp} follows directly from the above by Gr\"onwall's inequality for all $t < \tau_\rho$.

 It now only remains to show that under the given choice of $\alpha$, we have $\tau_\rho = \infty$. By the contraction of $T^\lambda$ \lref{l:tsitsiklis} and our choice of $\rho < \sigma_{\text{min}}/(2L_{DV})$ we write
 \begin{equ}
  \biggl\lVert\frac \d {\d\, t}w(t)\biggr\rVert_2 \leq \frac 1\alpha \|D\hatt V_{w(t)}\| \|T^\lambda \alpha V_{w(t)}-\alpha V_{w(t)}\|_\mu \leq \frac2\alpha \|D\hatt V_{w(0)}\|\|\alpha V_{w(t)} - V^*\|_\mu\,.
 \end{equ}
 Integrating the above and combining with the result from \eref{e:final1} in the previous paragraph we have
 \begin{equs}
  \|w(t)-w(0)\|_2 &\leq \frac2\alpha \|D\hatt V_{w(0)}\| \|\alpha V_{w(0)} - V^*\|_0 \int_0^t \exp\pq{ \frac{\gamma-1}{2\kappa^2}s} \, \d s\\& \leq \frac{4 \kappa^2 }{\alpha (1-\gamma)} \|D\hatt V_{w(0)}\| \|\alpha V_{w(0)} - V^*\|_0\,.\label{e:lasteq}
 \end{equs}
 Given that $\|\alpha V_{w(0)} - V^*\|_0\leq 2 \alpha M$, the above quantity is bounded by $\rho$ and therefore $\tau_\rho = \infty$, as desired.

 Finally, from \eref{e:lasteq} we see that if $\|\hatt V_{w(0)}\|_0 \leq C \alpha^{-1}$ then $\|w(t)-w(0)\|_2 \leq \frac{4 \kappa^2 }{\alpha (1-\gamma)} \|D\hatt V_{w(0)}\|(C + M \alpha_0) = \mathcal O(\alpha^{-1})$ for all $t > 0$.
\end{proof}

\subsection{Under-parametrized regime}



\begin{replemma}{l:forwardinvariance}There exists $\delta > 0$ and $\alpha_0 > 0$ such that the ball $\BB_\delta^0(0) \subseteq \obar \FF_0$
  is forward invariant and forward complete \abbr{wrt} the dynamics of \eref{e:tduavscaled} for all $\alpha > \alpha_0$.
\end{replemma}

\begin{proof}\textbf{of \lref{l:forwardinvariance}}
 We define the Lyapunov function $\bar U(f) := \frac 12\|f\|_0^2$, whose sublevel sets are $\BB_\delta^0(0)$. We prove forward invariance of such sets by showing that, on their boundary (\ie on the sphere $S_\delta^{r-1} \subset \bar\FF_0$ of radius $\delta$), $\bar U(f)$ decreases along trajectories of \eref{e:tduavscaled}  for $\alpha$ large enough.

 Noting that $S_\delta^{r-1} \subset \obar \FF_0$ upon taking $\delta$ small enough, we differentiate $\bar U(\obar V_{w(t)})$ \abbr{wrt} time for $w(t)$ obeying \eref{e:tduavscaled} at points $\obar V := \obar V_{w(t)} \in S_\delta^{r-1}$:
 \begin{equs}
  \frac{\d}{\d t} \bar U(\obar V) &= \frac{1}\alpha \dtp{\obar V, \bar g_{w(t)}^{-1} D \psi_{\obar V}^{-1} \Gamma(T^\lambda \alpha \psi^{-1}(\obar V) - \alpha \psi^{-1}(\obar V)) }_0\\
  &= \frac{1}\alpha \dtp{ \obar V, (D \psi_{\obar V}^{-1})^{\top} \Gamma (T^\lambda \alpha \psi^{-1}(\obar V) - \alpha \psi^{-1}(\obar V)) } + R_g(\obar V) \\
  & = \frac{1}\alpha \dtp{D \psi_{\obar V}^{-1} \obar V,  \bar r^\lambda +  \alpha (\gamma P^\lambda - \indicator) \psi^{-1}(\obar V) }_{\mu} + R_g(\obar V)\\
  & \leq \dtp{D \psi_{\obar V}^{-1} \obar V, (\gamma P^\lambda - \indicator) \psi^{-1}(\obar V) }_{\mu} + \frac{1}\alpha \|D \psi_{\obar V}^{-1} \obar V\|_\mu \|\bar r^\lambda\|_\mu +  |R_g(\obar V)|\,. \label{e:dU}
\end{equs}
 where we have defined $R_g(\obar V) := \frac{1}\alpha \dtp{\obar V, \tilde g_{w(t)} (D \psi_{\obar V}^{-1})^\top \Gamma(T^\lambda \alpha \psi^{-1}(\obar V) - \alpha \psi^{-1}(\obar V)) }$ for $\tilde g_w$ from \lref{l:perturbation}. We now proceed to bound the last two terms on the \abbr{rhs} from above. The second term is of order $\alpha^{-1}$ and therefore goes to $0$ for $\alpha \to \infty$ while for the last one we have that, by the equivalence of the norms $\|\,\cdot\,\|_\mu$ and $\|\,\cdot\,\|_2$,
 \begin{equs}
  |R_g(\obar V)| &\leq \frac 1 \alpha \|\obar V\|_2 \|\tilde g_{w(t)}\| \|(D \psi_{\obar V}^{-1})^\top \Gamma\pq{\bar r^\lambda + (\gamma P^\lambda - \indicator) \alpha \psi^{-1}(\obar V) }\|_2\\
  & \leq \frac 1 \alpha \|\obar V\|_2 \|\tilde g_{w(t)}\| \|(D \psi_{\obar V}^{-1})^\top \Gamma \bar r^\lambda\| + \|\obar V\|_2 \|\tilde g_{w(t)}\|\|(D \psi_{\obar V}^{-1})^\top \Gamma (\gamma P^\lambda - \indicator)  \psi^{-1}(\obar V) \|_2 \\
  &\leq \alpha^{-1}C + \epsilon_R(\delta)\|\obar V\|_\mu^2\,.\label{e:pertmetric}
 \end{equs}
 for a constant $C$ bounded by the norm of all operators and, by \lref{l:perturbation} a positive function $\epsilon_R(\delta)$ with $\lim_{\delta \to 0} \epsilon_R(\delta) = 0$.
 By the bounds established above and the fact that $\|\obar V\|_\mu \geq \kappa^{-1} \delta$ for $\obar V \in S_\delta^{r-1} \subset \obar \FF_0$ it is sufficient to show that the first term in \eref{e:dU} satisfies
 \begin{equ}\label{e:partialcontraction}
  \dtp{D \psi_{\obar V}^{-1} \obar V, (\gamma P^\lambda - \indicator) \psi^{-1}(\obar V) }_{\mu}  \leq - \epsilon \|\obar V\|_\mu^2\,,
\end{equ} for $\delta$ small enough and a constant $\epsilon > 0$
independent of $\delta$. We Taylor-expand $\psi^{-1}$ around the
origin, denoting the second order remainder of that expansion by
$R_2(\,\cdot\,,\,\cdot\,)$, and since $\psi^{-1}(\obar V_0) = 0$ we
have,
 \begin{equs}
  \dtp{D \psi_{\obar V}^{-1} \obar V, (\gamma P^\lambda - \indicator) \psi^{-1}(\obar V) }_{\mu} = \dtp{D \psi_{\obar V}^{-1} & \obar V, (\gamma P^\lambda - \indicator) D \psi_{0}^{-1} \obar V}_{\mu} \\ & +  \dtp{D \psi_{\obar V}^{-1} \obar V, (\gamma P^\lambda - \indicator) R_2(\obar V, \obar V)}_\mu\,,\label{e:temp4}
\end{equs}
where we have introduced the short hand notation
$D \psi_{0}^{-1} = D \psi_{\obar V_0}^{-1}$.  By the Lipschitz
smoothness of $\psi^{-1}(\,\cdot\,)$ \citep{lee03} we can bound the
norm of the second term from above as
 \begin{equ}\label{e:lipschitz}
  \dtp{D \psi_{\obar V}^{-1} \obar V, (\gamma P^\lambda - \indicator) R_2(\obar V, \obar V)}_\mu \leq 2 \|D \psi_{\obar V}^{-1} \obar V\|_\mu \|R_2(\obar V, \obar V)\|_\mu \leq 2  L_{D \psi^{-1}}\|D \psi_{\obar V}^{-1}\| \|\obar V\|_\mu^3\,.
 \end{equ}
 For the first term in \eref{e:temp4} we can also expand
 $D \psi_{\obar V}^{-1} =  D \psi_{0}^{-1} + \tilde R_2(\obar V, \,\cdot\,)$ ,
 and by applying a similar bound as \eref{e:lipschitz} we obtain that
 \begin{equ}\label{e:1}
  \dtp{D \psi_{\obar V}^{-1} \obar V, (\gamma P^\lambda - \indicator) D \psi_{0}^{-1} \obar V}_{\mu} \leq \dtp{D \psi_{0}^{-1} \obar V, (\gamma P^\lambda - \indicator) D \psi_{0}^{-1} \obar V}_{\mu} + 2  L_{D \psi^{-1}}\|D \psi_{0}^{-1}\| \|\obar V\|_\mu^3\,.
 \end{equ}
 The second term of the above equation being $\mathcal O (\|\obar V\|^3)$, we now consider the first one. By the nonexpansion of $P$ in $\|\,\cdot\,\|_\mu$ proven in \lref{l:tsitsiklis} we have
 \begin{equs}
  \dtp{D \psi_{0}^{-1} \obar V, (\gamma P^\lambda - \indicator) D \psi_{0}^{-1}\obar V }_{\mu} & \leq \gamma\|D \psi_{0}^{-1} \obar V\|_\mu \| P^\lambda D \psi_{0}^{-1} \obar V\|_\mu-\|D \psi_{0}^{-1} \obar V\|_\mu^2\\
  &\leq (\gamma - 1) \|D \psi_{0}^{-1} \obar V\|_\mu^2 \leq (\gamma - 1) 
  (\sigpsi)^2\|\obar V\|_\mu^2\,,\label{e:2}
 \end{equs}
 where $\sigpsi$ denotes the smallest singular value of $D \psi^{-1}$ in $\BB_\delta^0(0)$.
 Combining \eref{e:lipschitz}, \eref{e:1} and \eref{e:2} we finally obtain
 \begin{equ}
  \dtp{D \psi_{\obar V}^{-1} \obar V, (\gamma P^\lambda - \indicator) \psi^{-1}(\obar V) }_{\mu} \leq \|\obar V\|_\mu^2((\gamma - 1)(\sigpsi)^2 + C' \kappa^{-1} \|\obar V\|_0)\,,\label{e:finalcontraction}
 \end{equ}
 for $C' = 2  L_{D \psi^{-1}}(\|D \psi_{0}^{-1}\| + \|D \psi_{\obar V}^{-1}\|)$ and recalling that $\kappa$ is the equivalence constant between the norms $\|\,\cdot\,\|_\mu$ and  $\|\,\cdot\,\|_0$  in $\obar \FF_0$.
 \footnote{We recall that by the construction of the mappings $\psi, \phi, \pi_r$ and by our assumption in \tref{t:under-parametrized} the metric tensor $\bar g_t$ has full rank on $\FF_0$ and being the latter set compact its eigenvalues are uniformly bounded from below. At the same time, we can equip $\obar \FF_0$ with the metric induced by $\Gamma$ by restricting it to its first $r$ elements, which are uniformly bounded from below. Hence, the two metrics are equivalent on this space for some equivalence constant $\kappa$.} Therefore, choosing $\delta$ small enough we obtain \eref{e:partialcontraction} and conclude the proof of forward invariance.

 By boundedness of $\BB_\delta^0(0)$ in $\obar \FF_0$, forward completeness follows directly from forward invariance.
\end{proof}

\begin{replemma}{l:convexity} There exists $\ell > 0$, $\delta > 0$ and $\alpha_0 > 0$ such that for all $\alpha > \alpha_0$ and all geodesics $\gamma(s)$ contained in the ball $\BB_\delta^0(0) \subseteq \obar \FF_0$, the function
  \begin{equ}
    {\dtp{\gamma'(s),X(\gamma(s))}_{0} - \ell s \dtp{\gamma'(0),\gamma'(0)}_{0}}\,,
  \end{equ}
  is strictly decreasing in $s$.
\end{replemma}
%

\begin{proof}\textbf{of \lref{l:convexity}}
 To simplify the notation and the forthcoming computation, we prove the differential version of the desired result, \ie we show that there exists $\ell > 0$ such that
 \begin{equ}\label{e:convexity2}
  \frac{\d}{\d\,s}\pq{\dtp{\gamma'(s),X(\gamma(s))}_{0} - \ell s \dtp{\gamma'(0),\gamma'(0)}_{0}} < 0\,.
 \end{equ}
 The above expression exists almost everywhere by Lipschitz continuity of the terms to be differentiated. When this is not the case, we must interpret this derivative in the sense of distributions. We will highlight the steps where this could be necessary as we go along the proof.

 In our case, $X$ is the \abbr{rhs} of \eref{e:vf} mapped through $\psi$ onto $\obar \FF_0$, \ie
 \begin{equ}
  X (\gamma(s)) = - \frac 1\alpha \bar g_{\gamma(s)}^{-1} (D \psi_{\gamma(s)}^{-1})^\top  \Gamma ( T^\lambda\alpha  \psi^{-1}(\gamma(s)) - \alpha \psi^{-1}(\gamma(s)))\,.
 \end{equ}
 We are going to consider the "flattened" manifold obtained by the maps $\phi$ and $\psi$ equipped with the metric $\bar g_0$.
 In this space, geodesics have the form  $\gamma(s) = v_1 + s \Delta v $ where $\Delta v := v_2-v_1$ for $v_1,v_2 \in \obar \FF_0$ and their derivative is $\gamma'(s) = \Delta v$. Consequently \eref{e:convexity2} reads
 \begin{equs}\label{e:condition}
  {\dtp{\Delta v,\frac{\d }{\d s} X(\gamma(s))}_{0} < \ell \|{\Delta v}\|_{0}^2}\,,
 \end{equs}
 where defining $\tilde g_{\gamma(s)} := \bar g_0 \bar g_{\gamma(s)}^{-1} - \indicator  $ as in \lref{l:perturbation} we have
 \begin{equs}
  \frac{\d }{\d s} X(\gamma(s))  &= \frac{\d }{\d s} \bar g_0 \bar g_{\gamma(s)}^{-1} (D \psi_{\gamma(s)}^{-1})^\top \Gamma (T^\lambda (\alpha \psi^{-1}(\gamma(s)))- \alpha \psi^{-1}(\gamma(s)))\\
  & = \frac{\d }{\d s}\bar X(\gamma(s)) + \tilde g_{\gamma(s)} \frac{\d }{\d s}\bar X(\gamma(s)) + D \tilde g_{\gamma(s)}(\bar X(\gamma(s)), \gamma'(s))\,.\label{e:dgX}
 \end{equs}
 for
 \begin{equ}
  \bar X(\gamma(s)) := (D \psi_{\gamma(s)}^{-1})^\top \Gamma (T^\lambda (\alpha \psi^{-1}(\gamma(s)))- \alpha \psi^{-1}(\gamma(s)))\,.
 \end{equ}
 We proceed by analyzing the first term in the above equation and leave the task of bounding the last two for later. Using $\partial_s \alpha \psi^{-1}({\gamma(s)}) =  \alpha D \psi_{\gamma(s)}^{-1} \gamma'(s) =  \alpha D \psi_{\gamma(s)}^{-1} \Delta v$ we have that
 \begin{equs}\label{e:split1}
  \frac{\d }{\d s} \bar X(\gamma(s))  = \frac 1 \alpha (D^2 \psi_{\gamma(s)}^{-1})^\top & (\Gamma (T^\lambda \alpha \psi^{-1}({\gamma(s)})- \alpha \psi^{-1}({\gamma(s)})), \Delta v )  \\ & +  (D \psi_{\gamma(s)}^{-1})^\top \Gamma \big[ D T^\lambda  D \psi_{\gamma(s)}^{-1} \Delta v -  D \psi_{\gamma(s)}^{-1} \Delta v\big]\,,
 \end{equs}
 where $(D^2 \psi_{\gamma(s)}^{-1})^\top $ denotes the inversion of the last two indices of the Hessian.
 We now proceed to consider the two terms in the sum above separately (multiplied by the scalar product of \eref{e:condition}), defining throughout $(TD)_s := \Gamma (T^\lambda \alpha \psi^{-1}({\gamma(s)})- \alpha \psi^{-1}({\gamma(s)}))$.
 For the first term we have:
 \begin{equ}\label{e:hessian}
  \frac 1 \alpha \dtp{\Delta v, D^2 \psi_{\gamma(s)}^{-1}  (TD_s, \Delta v )}_0  \leq  \|\Delta v\|_0^2 \|D^2 \psi_{\gamma(s)}^{-1}\pc{\alpha^{-1} \bar r^\lambda +  (\gamma P^\lambda - \indicator)\psi^{-1} \gamma(s)}\| \leq \epsilon' \|\Delta v\|_0^2\,,
 \end{equ}
 for any $\epsilon' > 0$ by using the linearity of the Hessian and bounding its operator norm of $\psi^{-1}$ on a compact space in $\FF_0$ while choosing $\alpha$ large enough and $\delta$ small enough, since $\gamma(s) \in \BB_\delta^0(0)$. Note that if $D \psi^{-1}$ is not differentiable, the above computation is to be understood in the sense of distributions.  

 We now focus on the second term of \eref{e:split1}. In this case we incorporate the operator $\Gamma$ in the inner product and write this term as
 \begin{equ}
  \dtp{D \psi_{\gamma(s)}^{-1}\Delta v,  D T^\lambda  D \psi_{\gamma(s)}^{-1} \Delta v }_\mu - \|D \psi_{\gamma(s)}^{-1}\Delta v \|_\mu^2\,.
 \end{equ}
 Now, by the contraction property of $T^\lambda$ onto the tangential space $\edit{{\mathcal T}_{\psi_{\gamma(s)}^{-1}}} \FF$ in the norm $\|\,\cdot\,\|_\mu$ we can write
 \begin{equs}
  \dtp{D \psi_{\gamma(s)}^{-1}\Delta v,  D T^\lambda  D \psi_{\gamma(s)}^{-1} \Delta v }_\mu  \leq \|D \psi_{\gamma(s)}^{-1}\Delta v\|_\mu \| P^\lambda D \psi_{\gamma(s)}^{-1}\Delta v\|_\mu \leq \gamma \|D \psi_{\gamma(s)}^{-1}\Delta v\|_\mu^2\,,
 \end{equs}
 so that
 \begin{equ}
  \dtp{D \psi_{\gamma(s)}^{-1}\Delta v,  D T^\lambda  D \psi_{\gamma(s)}^{-1} \Delta v }_\mu - \|D \psi_{\gamma(s)}^{-1}\Delta v \|_\mu^2\leq (\gamma-1) \|D \psi_{\gamma(s)}^{-1}\Delta v\|_\mu^2\label{e:partial2}\,.
 \end{equ}
 Denoting by $\sigma_{\text{\scriptsize max}}^{D\psi^{-1}},\sigma_{\text{\scriptsize min}}^{D\psi^{-1}}$ the largest and smallest, respectively, singular values of the map $D \psi^{-1}$ in $\BB_{\delta}^0(0)$ (which are bounded away from $0$ upon possibly making this set smaller), by nondegeneracy of $D \psi^{-1}$ and by the equivalence of the $\|\,\cdot\,\|_\mu$ and $\|\,\cdot\,\|_0$ norms on $\obar \FF_0$ we have that
 \begin{equ}
  \kappa^{-1} \sigpsi\|\Delta v\|_0\leq \|\Delta  v\|_\mu \sigpsi \leq \|D \psi_{\gamma(s)}^{-1}\Delta v\|_\mu\leq \|\Delta v\|_\mu\sigma_{ \text{\scriptsize max}}^{D\psi^{-1}} \leq \kappa\|\Delta v\|_0\sigma_{\text{\scriptsize max}}^{D\psi^{-1}}\,.
 \end{equ}
 Thus we have
 \begin{equ}\label{e:partial}
  \|D \psi_{\gamma(s)}^{-1}\Delta v \|_\mu^2 \geq \kappa^{-2} \Bigl(\sigpsi\Bigr)^2 \|\Delta v\|_0^2\,.
 \end{equ}

 Getting back to the last two terms in \eref{e:dU}, we immediately see
 from \lref{l:perturbation} that $\tilde g_{\gamma(s)}$ is a small,
 Lipschitz continuous perturbation. Hence, the
 product $$\dtp{\gamma'(s),\tilde g_{\gamma(s)} \bar X'(\gamma(s))}$$
 can be bounded from above similarly to \eref{e:pertmetric}, while the
 second order derivative in the third term of \eref{e:dgX} can be
 dealt with analogously to what is done in \eref{e:hessian}, giving
 terms $\epsilon''\|\Delta v\|_0^2$ and
 $\epsilon^{(3)}\|\Delta v\|_0^2$ respectively, both going to $0$ as $\delta \to 0$.

 Therefore, combining the above with \eref{e:hessian}, \eref{e:partial2} and \eref{e:partial} we have
 \begin{multline*}
   \dtp{\Delta v,\frac{\d }{\d t} \bar X(\gamma(s))}_{0} \leq
   \frac{\gamma - 1}{ \kappa^2} \Bigl(\sigpsi\Bigr)^2 \|\Delta
   v\|_0^2 + \pc{\sum_i^{3}\epsilon^{(i)}(\delta)} \|\Delta v\|_0^2\\
   \leq \frac{\gamma - 1}{ 2 \kappa^2} \Bigl(\sigpsi\Bigr)^2\|\Delta
   v\|_0^2\,.
 \end{multline*}
 This directly gives \eref{e:condition} by choosing $\ell$ large enough.
\end{proof}

The next lemma estimates the distance between the fixed point $\tilde V^*$ of the dynamics   \eref{e:tduavscaled} and $V^*$ given by \eref{e:extremum}, showing that it is close, for large values of $\alpha$ to the best linear model in the tangent space of $\FFw$ at $V_{w(0)}$, given by $\Pi_0 V^*$. We recall that the projection operator $\Pi_0$ onto the linear space spanned by the columns of $DV$ is given by \cite[Eq. (1)]{tsitsiklis}
\begin{equ}
 \Pi_0 W := \argmin_{\{ D V_{w(0)}\Delta w~:~\Delta w \in \Rr^p\}} \|D V_{w(0)}\Delta w - W\|_\mu = D V_{w(0)}(D V_{w(0)}^\top \Gamma D V_{w(0)})^{-1} D V_{w(0)}^\top \Gamma W \,,
\end{equ}
for all $W \in \FF$ where, if necessary, we interpret $(D V_{w(0)}^\top \Gamma D V_{w(0)})^{-1}$ as a pseudo-inverse.

\begin{lemma}\label{l:fixedpoint}
 Let $\tilde V^*$ be the fixed point of \eref{e:tduavscaled} and $V^*$ be the global fixed point of the \abbr{td} operator, given by \eref{e:extremum}. Then under the assumptions of \tref{t:under-parametrized} there exists constants  $\alpha_0 > 0$ and $C^* > 0$ (independent of $\alpha_0$), such that
 \begin{equ}\label{e:fakefixed}
  \|\tilde V^* - V^*\|_\mu < \frac{1-\lambda \gamma}{1-\gamma} \|\Pi_0 V^* - V^*\|_\mu + C^* \alpha^{-1}\,,
 \end{equ}
 where $\Pi_0$ is the projection operator onto $\edit{{\mathcal T}_{V({w(0)})}}\FFw$.
\end{lemma}

To prove the above result we compare the dynamics \eref{e:tduavscaled} to the dynamics of the model $V$ when \emph{linearized} at $w(0)$. In this case, the dynamics of the parameters is given by
\begin{equ}\label{e:barw}
 \frac{\d}{\d\,t} \bar w(t) = D V_{w(0)}^\top \Gamma (T^\lambda \VV_{\bar w(t)} - \VV_{\bar w(t)})\,,
\end{equ}
where $\VV \in \FF$ is the linear, tangent model of $V$ at $w(0)$ defined as
\begin{equ}\label{e:VV}
 \VV_{w} := V_{w(0)} + D V_{w(0)}(w-w(0))\,.
\end{equ}
We can also write the dynamics of the linear model as
\begin{equ}\label{e:tdVV}
 \frac{\d}{\d\,t} \VV_{\bar w(t)} := D V_{w(0)} \cdot D V_{w(0)}^\top \Gamma (T^\lambda \VV_{\bar w(t)} - \VV_{\bar w(t)})\,.
\end{equ}
Scaling the model as $\VV \to \alpha \VV$ and $t \to \alpha^{-1}t$ we obtain the analogue of \eref{e:tduavscaled}:
\begin{equ}\label{e:tdbarw}
 \frac{\d}{\d\,t}  \bar w (t) := \frac 1 \alpha  D V_{w(0)}^\top \Gamma (T^\lambda \alpha \VV_{\bar w(t)} - \alpha \VV_{\bar w(t)})\,.
\end{equ}
which in $\FF$ reads
\begin{equ}
 \frac{\d}{\d\,t}  \alpha \VV_{\bar w(t)} := D V_{w(0)} \cdot D V_{w(0)}^\top \Gamma (T^\lambda \alpha \VV_{\bar w(t)} - \alpha \VV_{\bar w(t)})\,.
\end{equ}

\begin{proof}\textbf{of \lref{l:fixedpoint}}
 {Recall from \cite[Lemma 6]{tsitsiklis} that for the linear
 value function approximation one has}
 \begin{equ}\label{e:tsitsiklis}
  \|\VV^* - V^*\|_\mu < \frac{1-\lambda \gamma}{1-\gamma} \|\Pi_0 V^* - V^*\|_\mu\,,
 \end{equ}
 where $\Pi_0$ is the projection on $\edit{{\mathcal T}_{V({w(0)})}}\FFw$ and $\VV^*$ is the unique fixed point of the dynamics \eref{e:tdVV} on that space.
 In light of this result, our task reduces to bounding the distance between the trajectories of the original {(\ie dynamics \eref{e:tduavscaled})} and the linearized  model {(\ie dynamics \eref{e:tdbarw})} by $C \alpha^{-1}$ for $C$ large enough. We do so in 3 main steps. First of all, we bound the maximal excursion of the models $\VV$ and $V$. Mapping both dynamics onto a common coordinate space, we then bound from above the distance between the two trajectories in this space by $\OO(\alpha^{-1})$. Finally, we map the dynamics back to the embedding space and show that the correction is again of the same order $\OO(\alpha^{-1})$.

 \paragraph{Bounding the maximal excursion.} To compare the dynamics of $\alpha V_{w(t)}$ and $\alpha \VV_{\bar w(t)}$ we map them to a common space. Recalling the definition of the maps $\phi, \pi_r, \psi$ from the proof of \tref{t:under-parametrized} we note that the first order expansion of $\psi$, maps $\edit{{\mathcal T}_{V({w(0)})}}\FFw$ to $\obar \FF_0$. Explicitly, for $\obar V \in \obar \FF_0$ and for $\Delta \VV \in \edit{{\mathcal T}_{V({w(0)})}}\FFw$ with $\|\Delta \VV\|_0$ small enough we have
 \begin{equ}\label{e:maps}
  \bar \psi (V_{w(0)} + \Delta \VV) := D \psi_0 \Delta \VV \qquad \text{and}\qquad  \bar \psi^{-1} (\obar V) = V_{w(0)} + D \psi_{0}^{-1} \obar V \in \edit{{\mathcal T}_{V({w(0)})}}\FFw\,.
 \end{equ}
 Now, we proceed to show that the dynamics of \eref{e:tduavscaled} and \eref{e:tdbarw}, mapped to $\FF_0$, do not exit a ball $\BB_{\delta}^0(0)$, when choosing $\delta = C/\alpha$ for $C$ large enough. We show this with the same strategy used for the proof of \lref{l:forwardinvariance}, \ie we show that $\bar U(f) := \frac 12 \|f\|_0^2$ decreases on $S_{\delta}^{r-1}(0)$ along the trajectories of interest ({note that $\delta$ is now much smaller than that used in \lref{l:forwardinvariance}}). We will start with the curved dynamics \eref{e:tduavscaled} and will then show that the same result follows, in a simpler setting, for \eref{e:tdbarw}.
 For $\obar V := \obar V_{w(t)} \in S_{\delta}^{r-1}(0)$ we start by bounding, as in \eref{e:dU},  the derivative
 \begin{equ}
  \frac{  \d\,}{\d\,t}\bar U(\obar V) \leq \dtp{D \psi_{\obar V}^{-1} \obar V, (\gamma P^\lambda - \indicator) \psi^{-1}(\obar V) }_{\mu} + \frac{1}\alpha \|D \psi_{\obar V}^{-1} \obar V\|_\mu \|\bar r^\lambda\|_\mu +  |R_g(\obar V)|\,. \label{e:dU2}
 \end{equ}
 Before bounding the above terms we recall that by Lipschitz smoothness of $\psi$ we have that
 \begin{equ}
  \|\psi^{-1}(\obar V)\| < \|V_{w(0)}\| + \|D\psi_0^{-1} \obar V\| + L_{D \psi^{-1}} \|\obar V \|^2\,.\label{e:lsmooth}
 \end{equ}
 Then, since $V_{w(0)} = 0$, similarly to \eref{e:dU} we have for the last term in \eref{e:dU2} that, for $\alpha$ large enough,
 \begin{multline*}
  |R_g(\obar V)| \leq  \|\tilde g_w\| \|\obar V\|_2\Bigl(\|\obar V\|_2\|(D \psi_{\obar V}^{-1})^\top \Gamma (\gamma P^\lambda - \indicator) \|(\|D\psi_0^{-1}\| + L_{DV}\|\obar V \|_2) \\ +\frac 1 \alpha \|(D \psi_{\obar V}^{-1})^\top \Gamma \bar r^\lambda\|_2 \Bigr)\,.
 \end{multline*}
 By the equivalence of the norms $\|\,\cdot\,\|_\mu$, $\|\,\cdot\,\|_2$ and $\|\,\cdot\,\|_0$ on $\Pi_r$  and since $\delta = C/\alpha$ we have that
 \begin{equ}\label{e:11}
  |R_g(\obar V)|  \leq  \|\tilde g_w\| \|\obar V\|_0^2 (K + 1) + \OO(\alpha^{-3})\,,
 \end{equ}
 upon increasing $C$ if necessary and defining $ K = \kappa_2^2\|(D \psi_{\obar V}^{-1})^\top \Gamma (\gamma P^\lambda - \indicator) \|\|D\psi_0^{-1}\|$ for $\kappa_2$ the equivalence constant between $\|\,\cdot\,\|_2$ and $\|\,\cdot\,\|_0$ {on $\Pi_r$}.
 The second term in \eref{e:dU2} can be bounded similarly to the above by the equivalence of norms:
 \begin{equ}\label{e:21}
  \frac{1}\alpha \|D \psi_{\obar V}^{-1} \obar V\|_\mu \|\bar r^\lambda\|_\mu \leq \|\obar V\|_0^2 \frac{\kappa^2\|D \psi_{\obar V}^{-1} \|\|\bar r^\lambda\|_\mu}{C}\,.
 \end{equ}
 The first term in \eref{e:dU2} can be treated identically to the proof of \lref{l:forwardinvariance} to obtain \eref{e:finalcontraction}. Changing the norm in \eref{e:finalcontraction} and combining it with \eref{e:11} and \eref{e:21} gives
 \begin{equ}
  \frac{  \d\,}{\d\,t}\bar U(\obar V) \leq \|\obar V\|_0^2\pc{\frac{\gamma - 1}{2\kappa^2}(\sigpsi)^2 + \frac{\kappa^2\|D \psi_{\obar V}^{-1} \|\|\bar r^\lambda\|_\mu}{C} +  \|\tilde g_w\| (K+1)} + \OO(\alpha^{-3})\,.
 \end{equ}
 Since $\gamma - 1 < 0$, we can choose $C$ large enough to make the second term in brackets smaller than $(\gamma - 1)/{12 \kappa^2}(\sigpsi)^2$. The same holds for the third term in brackets by \eref{e:perturbation}, and for the higher order term by taking $\alpha$ large enough, showing that
 \begin{equ}
  \frac{  \d\,}{\d\,t}\bar U(\obar V) \leq \frac{\gamma - 1}{4 \kappa^2} (\sigpsi)^2 \|\obar V\|_0^2 < 0\,,
 \end{equ}
 as desired. We note that the same reasoning with $L_{DV} = 0$ and $D \psi_{\obar V}^{-1} \equiv D \psi_0^{-1}$ yields an identical conclusion for the dynamics of $\VV$ in a ball of radius $\delta = C/\alpha$ for $C, \alpha$ large enough.
 Also, we note that combining the above computation with \eref{e:lipschitz} yields
 \begin{equs}
  \|D \psi_{\obar V}^{-1} \Gamma (T^\lambda \alpha \psi^{-1}(\obar V) - \alpha \psi^{-1}(\obar V))\|&\leq \|D \psi_{\obar V}^{-1} \Gamma \| (\|\bar r^\lambda \|+ \alpha (\gamma + 1) \|D \psi_0^{-1} \obar V\| + \alpha L_{D \psi_{\obar V}^{-1}}\|\obar V\|^2 )\\
  & \leq (\gamma + 1)\|D \psi_{\obar V}^{-1} \Gamma \| (\|D \psi_0^{-1}\|C + \|\bar r^\lambda \| + \OO(\alpha^{-1})) \\&\leq C_0\,,\label{e:boundontd}
 \end{equs}
 for $C_0$ large enough, where $D \psi_{\obar V}^{-1} \Gamma$ is considered as an operator mapping $\FF_0 \to \bar \FF_0$.

 \paragraph{Bounding the distance of trajectories.} The
 distance 
 between two trajectories with the same initial condition can be
 bounded by {$\OO(\alpha^{-2})$ using a similar argument as in
   \cite[Lemma B2]{cb182} for the present framework}. We include the
 proof of this lemma here as the assumptions are not identical and to
 make the paper self-contained, while we do not claim any improvement
 on that result. To enounce this result, we recall that $\sigpsi$
 denotes the smallest singular eigenvalue of $D\psi^{-1}$ in a ball
 $\BB_\delta^0(0)$, which is bounded away from $0$ for $\delta$ small
 enough. Similarly, we recall that
 $\bar g_t^{-1} \succeq \sigg \indicator$ for $\sigg > 0$ in
 $\BB_\delta^0(0)$ for $\delta$ small
 enough. 
 \begin{lemma}\label{l:cb}
  Let $\obar V_t$, $\obar \VV_t$ in $\obar \FF_0$ be solutions of
  \begin{equs}
   \frac{  \d}{\d\,t}\obar V_t & = \bar g_t^{-1} (D \psi_{\obar V_t}^{-1})^\top \Gamma (T^\lambda \alpha \psi^{-1}(\obar V_t)-\alpha \psi^{-1}(\obar V_t) )\,,\\
   \frac{  \d}{\d\,t}\obar \VV_t &= \bar g_0^{-1}(D \psi_0^{-1})^\top \Gamma (T^\lambda \alpha \bar \psi^{-1}(\obar \VV_t)-\alpha \bar \psi^{-1}(\obar \VV_t) )\,.
  \end{equs}
  Then defining $K := \sup_{t>0}\|(\bar g_t^{-1}- \bar g_0^{-1}) (D \psi_{\obar V_t}^{-1})^{\top}\Gamma (T^\lambda \alpha \psi^{-1}(\obar V_t)-\alpha \psi^{-1}(\obar V_t))\| $ and $\beta := \frac{1-\gamma}{\kappa^2}(\sigpsi)^2$ we have that
  \begin{equ}
   \sup_{t > 0} \|\obar V_t - \obar \VV_t\|_0 \leq \frac 1 \alpha \frac{2K}{\beta} \,.
  \end{equ}
 \end{lemma}

 \begin{proof}\textbf{of \lref{l:cb}}
  We define the function $h(t) := \frac{1}{2}\|\obar V_t - \obar \VV_t\|_0^2$, take its time derivative
  \begin{equs}
   h'(t)  &= \dtp{\obar V_t' - \obar \VV_t', \obar V_t- \obar \VV_t}_0 \,,
  \end{equs}
  and defining
  \begin{equs}
   (TD)_t & := T^\lambda \alpha \psi^{-1}(\obar V_t)-\alpha \psi^{-1}(\obar V_t)\,,\\
   \TD_t & := T^\lambda \alpha \bar \psi^{-1}(\obar \VV_t)- \alpha \bar \psi^{-1}(\obar \VV_t)\,,
  \end{equs} we evaluate (for simplicity of notation, we introduce the short hand $D \psi_t^{-1} := D \psi_{\obar V_t}^{-1}$ for the rest of the proof)
  \begin{equs}
   \obar V_t' - \obar \VV_t' &=  \frac1{\alpha}{\bar g_t^{-1} (D \psi_t}^{-1})^\top \Gamma (TD)_t- \frac1{\alpha}\bar g_0^{-1} (D \psi_0^{-1})^\top \Gamma \TD_t \\
   & \leq   \frac1{\alpha} \pq{\bar g_0^{-1}  (D \psi_t^{-1})^\top \Gamma (TD )_t-\bar g_0^{-1}  (D \psi_0^{-1})^\top \Gamma \TD_t} \label{e:almostlast1}\\
   &\qquad +  \frac1{\alpha} \pq{\bar g_t^{-1} (D \psi_t^{-1})^\top \Gamma (TD)_t- \bar g_0^{-1}  (D \psi_t^{-1})^\top \Gamma (TD)_t}\,.\label{e:almostlast2}
  \end{equs}
  We look at the two terms on the \abbr{rhs} separately and obtain, for \eref{e:almostlast1}
  \begin{equs}
   \frac1{\alpha} \dtp{\bar g_0^{-1} & (D \psi_t^{-1})^\top \Gamma (TD )_t-\bar g_0^{-1} (D \psi_0^{-1})^\top \Gamma \TD_t,  \obar V_t- \obar \VV_t}_0 \label{e:last}
   \\ & = \frac1{\alpha} \dtp{(D \psi_t^{-1})^\top \Gamma (TD)_t -(D \psi_0^{-1})^\top \Gamma \TD_t ,  \obar V_t- \obar \VV_t}\\
   & = \frac1{\alpha} \dtp{ {(TD)_t-\TD_t},  D \psi_0^{-1} (\obar V_t- \obar \VV_t)}_\mu\label{e:last1}\\
   & \qquad + \frac1{\alpha} \dtp{(D \psi_t^{-1}- D\psi_0^{-1})^\top \Gamma (TD)_t,  \obar V_t- \obar \VV_t}\,.\label{e:last2}
  \end{equs}
  We immediately see that by Lipschitz smoothness of $\psi^{-1} $ and the equivalence of $\|\,\cdot\,\|_2$ and $\|\,\cdot\,\|_0$ norms on $\Pi_r$ and \eref{e:boundontd}, for \eref{e:last2} we have
  \begin{equ}
   \frac1{\alpha} \dtp{(D \psi_t^{-1}- D\psi_0^{-1})^\top \Gamma (TD)_t,  \obar V_t- \obar \VV_t} \leq \frac1{\alpha} L_{D \psi^{-1}}\|\obar V_t\|_2\|\Gamma (TD)_t\|\|\obar V_t- \obar \VV_t\|_2 \leq \frac{C_1}{\alpha^2} \sqrt{2 h(t)}\,,\label{e:31}
 \end{equ}
 by choosing $C_1$ large enough.
  For \eref{e:last1} by the definition of $\psi$ we have
  \begin{equs}
   (TD)_t-\TD_t& = T^\lambda \alpha \psi^{-1}(\obar V_t)- T^\lambda \alpha \bar \psi^{-1}(\obar \VV_t) - \alpha ( \psi^{-1}(\obar V_t) - \bar \psi^{-1}(\obar \VV_t) )\\
   & = \alpha(P^\lambda-\indicator)( \psi^{-1}(\obar V_t) - \bar \psi^{-1}(\obar \VV_t))\,,
  \end{equs}
  and hence, by \eref{e:lsmooth} we have
  \begin{equs}
   \frac1{\alpha} \dtp{ {(TD)_t-\TD_t},  D \psi_0^{-1} (\obar V_t- \obar \VV_t)}_\mu &\leq \dtp{(P^\lambda-\indicator)( \psi^{-1}(\obar V_t) - \bar \psi^{-1}(\obar \VV_t)), D \psi_0^{-1} (\obar V_t- \obar \VV_t)}_\mu \\
   & \leq \dtp {(P^\lambda-1) D \psi_0^{-1}(\obar V_t- \obar \VV_t),D \psi_0^{-1}(\obar V_t- \obar \VV_t)  }_\mu\\
   & \quad + L_{D \psi^{-1}}\|\obar V_t\|_\mu^2 \|D \psi_0^{-1}\|\|\obar V_t- \obar \VV_t\|_\mu\,.
  \end{equs}
  Defining $\beta := \frac{1-\gamma}{\kappa^2}(\sigpsi)^2$, the first term from above can be bounded as in \eref{e:2} to obtain
  \begin{equ}
   \dtp {(P^\lambda-1) D \psi_0^{-1}(\obar V_t- \obar \VV_t),D \psi_0^{-1}(\obar V_t- \obar \VV_t)  }_\mu \leq -\beta h(t)\,,\label{e:32}
  \end{equ} while for the second by our choice of $\delta = C/\alpha$ we have
  \begin{equ}\label{e:33}
   L_{D \psi^{-1}}\|\obar V_t\|_\mu^2 \|D \psi_0^{-1}\|\|\obar V_t- \obar \VV_t\|_\mu \leq \frac{C^2}{\alpha^2}\kappa L_{D \psi^{-1}} \|D \psi_0^{-1}\| \sqrt{2 h(t)} \,.
  \end{equ}
  Finally, combining \eref{e:31}, \eref{e:32} and \eref{e:33} we have
  \begin{equ}
   \eref{e:last} \leq - \beta  h(t) + \frac{C_2}{\alpha^2} \sqrt{2h(t)}\,, \label{e:41}
  \end{equ}
  where $C_2 := C_1 + C^2 \kappa L_{D \psi^{-1}} \|D \psi_0^{-1}\|$.

  We now consider \eref{e:almostlast2}. Here by the definition of $K$ we have
  \begin{equs}
   \frac1{\alpha} \dtp{(\bar g_t^{-1}- \bar g_0^{-1}) D \psi_t^{-1} \Gamma & (TD)_t,\obar V_t- \obar \VV_t}_0  \leq \frac K{\alpha} \|\obar V_t- \obar \VV_t\|_0 = \frac K{\alpha} \sqrt{2 h(t)}\,.
  \end{equs}
  Combining the above with \eref{e:41} we finally obtain
  \begin{equ}
   h'(t) \leq - \beta   h(t) + \frac{K}{\alpha} \sqrt{2h(t)} + \frac{C_2}{\alpha^2} \sqrt{2h(t)} \leq - \beta   h(t) + \frac{2K}{\alpha} \sqrt{h(t)}\,,
  \end{equ}
  for $\alpha$ large enough.  The above expression is negative as soon
  as $h(t) > 4 K^2/(\alpha\beta)^2$. Therefore, because $h(0) = 0$, we must
  have that $h(t) \leq 4 K^2/(\alpha\beta)^2$ for all $t>0$, \ie
  \begin{equ}
   \|\obar V_t - \obar \VV_t\|_0< \frac 1 \alpha \frac{2 K}{\beta} \qquad \text{for all } t>0\,,
  \end{equ}
  as claimed.
 \end{proof}

To achieve the claimed $\OO(\alpha^{-2})$ bound, we observe that $K$ in the above Lemma can be chosen $\OO(\alpha^{-1})$ by the Lipschitz continuity of $\bar g_t^{-1}$. Indeed, since we chose $\|\obar V\|_0 = C/\alpha$, by \eref{e:boundontd} we have that
 \begin{equ}
  K \leq \sup_{t > 0} \|\Gamma (TD)_t\| \|D \psi_{\obar V_t}^{-1}\| L_{\bar g_0^{-1}} \|\obar V\|_0 \leq C_0 \sigma_{\text{\scriptsize max}}^{D\psi^{-1}} L_{\bar g_0^{-1}} \frac C \alpha \leq \frac{\beta}2\frac{K'}{\alpha}\,,
 \end{equ}
 for $K'$ large enough, and therefore
 \begin{equ}
  \|\obar V_t - \obar \VV_t\|_0< \frac{K'}{\alpha^2} \qquad \text{for all } t > 0\,.\label{e:lastlast}
 \end{equ}

 \paragraph{Mapping to the embedding space.} We conclude the proof  by mapping back to the original space, where we have
 \begin{equs}
  \sup_{t > 0}\|\VV_t - V_t\|_\mu & = \sup_{t}\|\alpha \psi^{-1} (\obar V_t) - \alpha \bar \psi^{-1}( \obar \VV_t)\|_\mu\\
  & \leq \sup_{t}\alpha\pc{ \|D \psi_0^{-1}(\obar V_t - \obar \VV_t)\|_\mu + L_{D \psi^{-1}}\|\obar V_t \|_\mu^2}\\
  & \leq \alpha\pc{ \kappa \|D \psi_0^{-1}\| \sup_t \|\obar V_t - \obar \VV_t\|_0 + \kappa^2 L_{D \psi^{-1}}\sup_t\|\obar V_t \|_0^2} \,.
 \end{equs}
 Then, letting $\VV^*$ be the fixed point of \eref{e:tdVV} (unique and attracting by \cite{tsitsiklis}), by our choice of $\delta = C / \alpha$, \eref{e:tsitsiklis} and \eref{e:lastlast} we have that
 \begin{equs}
  \|\tilde V^*- V^*\|_\mu  & \leq \|\VV^* - V^*\|_\mu  + \sup_{t > 0}\|\VV_t - V_t\|_\mu \\
  & \leq \frac{1-\gamma \lambda}{1-\gamma} \|\Pi_0 V^* - V^*\|_\mu +   \frac 1 \alpha (\kappa \|D \psi_0^{-1}\|K' + \kappa^2 L_{D \psi^{-1}}C^2)\,,
 \end{equs}
 as claimed.
\end{proof}

\section{Supplementary proofs in the mean field regime}

\subsection{Proof of convergence of the particle system}

\begin{repproposition}{t:particles}
  Let \aref{a:mf} hold and let $w_t^{(N)}$ be a solution of \eref{e:tduc} with initial condition $w_0^{(N)} \in \mathcal W = \Omega^N$. If $\nu_0^{(N)}$ converges to $\nu_0 \in \mathcal P_2(\Omega)$ in Wasserstein distance $W_2$ as $N \to \infty$ then $\nu_t^{(N)}$ converges, for every $t>0$, to the unique solution $\nu_t$ of \eref{e:mfpde}.
\end{repproposition}

\begin{proof}
As anticipated in the main text, the proof of this result is mainly standard. Therefore, we refer to \cite[Theorem 2.6]{ChizatBach18} for a detailed proof of an analogous result in the supervised setting and we proceed to highlight the steps that need to be adapted in that proof to cover the given setting. The key estimate needed to prove propagation of chaos concern the Lipschitz continuity of the transport vector field in \eref{e:mfpde}, which yields the desired result by Gronwall inequality. We note that, in order to establish such estimates, the authors combine the assumed regularity of the nonlinearity $\psi$ with the ones of the energy functional $R~:~\mathcal F \to \mathbb R$. Therefore, since our assumptions on $\psi$ coincide with those of \citet{ChizatBach18}, it is sufficient to recover the desired result to show that the \abbr{td} error $\delta(\nu)$, playing the role of $dR$ from \citet{ChizatBach18} in the present context, is Lipschitz continuous on bounded sets and bounded on sets that are forward invariant with respect to the dynamics. The former property, however, follows immediately by the \emph{linearity} of $\delta(\nu)$ in $\nu$, and the assumed regularity of the nonlinearity $\psi$.

We conclude the proof by establishing the latter property, corresponding to ``boundedness of $dR$ on sublevel sets'' in \cite{ChizatBach18} used to prove Lipschitz continuity of the vector field in $\nu$. This directly results from \lref{l:bounded} since boundedness of $\|V_\nu\|_\mu$ implies, by compactness of $\mathcal S$, the boundedness of $\delta(\nu)$ along the trajectories of \eref{e:mfpde} as required.
\end{proof}

\begin{replemma}{l:bounded}
  For any $\nu_0$ with $\int \omega_0^2 \nu_0(\d \omega) \leq \infty$ there exists $C_V>0$ such that, for any $t>0$ we have $\|V_{\nu_t}\|_\mu < C_V$.
\end{replemma}
\begin{proof}
   We consider the evolution of the quantity
\begin{equ}
  \|(\nu_t)_0\|_{2}^2 := \int_\Omega \omega_0^2 \nu_t(\d \omega)\,,
\end{equ}
for which we have, by homogeneity of $\psi$ and integrating by parts,
\begin{equs}
  \frac 1 2 \frac \d {\d t} \|(\nu_t)_0\|_{2}^2  = \frac 1 2 \int_\Omega \omega_0^2   \frac \d {\d t} \nu_t(\d \omega) = \int_\Omega\int_{\mathcal S} \omega_0   \phi(s;\bar \omega )\delta(\nu_t) \nu_t(\d \omega) \mu(\d s) = \int_{\mathcal S} V(s)\delta(\nu_t) \mu(\d s)\,.
\end{equs}
Now, recalling the definition of the metric tensor $\Gamma$ induced by the measure $\mu$ we define throughout the positive definite operator $A := \Gamma (1-\gamma P)$. Then, whenever
\begin{equ}\label{e:condition2}
\|V_\nu\|_\mu>(1-\gamma)^{-1}\|(1-\gamma P)V^*\|_\mu+\epsilon=: B + \epsilon
\end{equ}
we have
\begin{equs}
  |V_\nu A V^*| &\leq \|V_\nu\|_\mu \|(1-\gamma P) V^*\|_\mu < (1-\gamma)\|V_\nu\|_\mu^2 - \epsilon (1-\gamma)\|V_\nu\|_\mu\\
  &<\|V_\nu\|_\mu^2 - \gamma\|V_\nu\|_\mu^2 - C_A \leq V_\nu A V_\nu - C_A
\end{equs}
for $C_A = \epsilon (1-\gamma)((1-\gamma)^{-1}\|AV^*\|_\mu+\epsilon)$, where in the last inequality we have combined $\|PV_\nu\|_\mu\leq \|V_\nu\|_\mu$ from \citet[Lemma 1]{tsitsiklis} with Cauchy-Schwarz inequality.
In light of the above, writing $\delta(t) = (r - (1-\gamma P) V_\nu) = (1-\gamma P) (V^* - V_\nu) $ we have
\begin{equ}\label{e:decayl2}
  \frac 1 2 \frac \d {\d t} \|(\nu_t)_0\|_{2}^2 < - \pc{V_\nu A V_\nu - \pd{V_\nu A V^*}} < - {C_A}<0\,,
\end{equ}
so that the norm $\|(\nu_t)_0\|_{2}^2$ must decrease whenever \eref{e:condition2} holds, for any $\epsilon>0$.
Furthermore, we note that by the boundedness of $\phi$, \ie by $\phi(s;\bar \omega)<C_\phi$ we have
\begin{equ}\label{e:l2bound}
  \|V_{\nu}\|_\mu^2 = \int_{\mathcal S} \pc{\int_{\Omega} \omega_0 \phi(s;\bar \omega) \nu(\d \omega)}^2 \mu(\d s) \leq C_\phi^2 \int_\Omega \omega_0^2 \nu(\d \omega) = C_\phi^2 \|(\nu_t)_0\|_{2}^2\,,
\end{equ}
Combining \eref{e:decayl2} and \eref{e:l2bound} directly implies that the space $\{\nu~:~\|(\nu)_0\|_2 \leq (B+\epsilon)/C_\phi\}$ (which contains $\{\nu~:~\|V_\nu\| \leq B+\epsilon\}$) is globally attractive with respect to the dynamics \eref{e:mfpde}. In other words, this means on one hand that if $\|(\nu_0)_0\|_2> (B+\epsilon)/C_\phi$ then we must have
$\|V_{\nu_t}\|_\mu < C_\phi \|(\nu_t)_0\|_2< C_\phi\|(\nu_0)_0\|_2$ for all $t\geq 0$. On the other hand if
 $\|(\nu_0)_0\|_2\leq  (B+\epsilon)/C_\phi$ there cannot exist $t\geq 0$
 such that $\|(\nu_t)_0\|_2 = 2(B+\epsilon)/C_\phi$, directly implying that $\|V_{\nu_t}\|_\mu < 2(B+\epsilon)$.
%
%
Therefore, for any $V^*$, resulting from the chosen reward function $r$ and transition operator $P$, setting $C_V =\max\{2(B+\epsilon),C_\phi\|(\nu_0)_0\|_2\}$ we must have $\|V_{\nu_t}\|_\mu\leq C_V$ for all $t\geq 0$ as required.
\end{proof}

\subsection{\edit{Proof of \tref{t:mf}}}

\begin{reptheorem}{t:mf}
    Let \aref{a:mf} hold and $\nu_t$ given by \eqref{e:mfpde} converge
    to $\nu^*$, then $V_{\nu^*} = V^*$ $\mu$-a.e.
\end{reptheorem}

\newtheorem*{assumptionA}{Assumption B}

Before proceeding to prove \edit{\tref{t:mf}}, we state the alternative form of \aref{a:mf} a) in the case where $\Theta \neq \Rr^{m-1}$.
{
\begin{assumptionA} \label{a:mf2} Assume that $\omega = (\omega_0, \bar \omega) \in \Rr \times \Theta$ for $\Theta \subset  \Rr^{m-1}$ which is the closure of an bounded open convex set. Furthermore $\psi(s; \omega) = \omega_0 \phi(s;\bar \omega)$ where $\phi$ is bounded, differentiable and $D\phi$ is Lipschitz.
    {Also, for all $f \in \FF$ the regular values of the map {$\bar \omega \mapsto g_f(\bar \omega):= \dtp{f,\phi(\,\cdot\,; \bar \omega)}$} are dense in its range and $g_f(\bar \omega)$ satisfies Neumann boundary conditions (i.e., for all $\bar \omega \in \partial \Theta$ we have $d g_f(\bar \omega)(n_{\bar \omega}) = 0$ where $n_{\bar \omega} \in \Rr^{m-1}$ is the normal of $\partial \Theta$ at $\bar \omega$).
    }
\end{assumptionA}}
\edit{The proof of \tref{t:mf} is carried out in two steps. We first show in \sref{ss:1} that as a consequence of the expressivity of the activation function \aref{a:mf} b), the suboptimality of a fixed point is reflected in a nonvanishing transport vector field of \eref{e:mfpde}}\footnote{\edit{We remind that this could still result in a fixed point if the measure $\nu$ loses support where the transport vector field is nonzero.}}. \edit{Then, in \sref{ss:2} we use this partial result to bring the assumption of convergence towards a local minimizer to a contradiction.}

\subsubsection{The relation between vector field and TD error}\label{ss:1}
\edit{To state the first lemma towards the proof of the above result, we observe that  the $\omega_0$-component of the transport vector field in \eref{e:mfpde} reads:
\begin{equ}\label{e:gradient0}
  \pc{\int_{\SS\times \SS} \nabla_\omega \psi(s;\omega) \delta(s,s',\nu_t) P(s,\d s')\mu(\d s)}_0 = \int_{\SS}  \phi(s;\omega) \int_{\SS}\delta(s,s',\nu_t) P(s,\d s')\mu(\d s) \,.
  \end{equ}
  We note that the above is a function of $\bar \omega'$ only. In the following lemma we relate the optimality of a fixed point of \eref{e:mfpde} to \eref{e:gradient0}.}
\edit{\begin{lemma}\label{p:mf}
    Let \aref{a:mf} hold and $\nu$ be such that the first component of the vector field in \eref{e:mfpde} vanishes a.e., \ie
    the condition
        \begin{equ}\label{e:fixedpoint}
      {\int_{\mathcal S^2} \pc{r(s,s') + \gamma \int_{\Omega} \psi(s';\omega ) \nu(\omega) \d \omega  - \int_{\Omega} \psi(s;\omega ) \nu(\omega) \d \omega}  \phi(s,\bar \omega') P(s, \d s')}\mu(\d s) = 0
        \end{equ}
    holds $\bar \omega'$-almost everywhere in $\Theta$.
    Then
    we have  $V_{\nu} = V^*$ $\mu$-a.e..
\end{lemma}
\begin{proof}\textbf{of \lref{p:mf}}
  The result follows immediately by \aref{a:mf} a)-b): Assuming that \eref{e:fixedpoint} holds Lebesgue-a.e. in $\Theta$, by the assumed continuity of $\phi$ in $\bar \omega'$ combined with the expressivity of $\phi$ \aref{a:mf} b) we must have that
  \begin{equ}
    {\int_{\mathcal S} \pc{r(s,s') + \gamma \int_{\Omega} \psi(s';\omega ) \nu(\omega) \d \omega  - \int_{\Omega} \psi(s;\omega ) \nu(\omega) \d \omega}  P(s, \d s')} = 0\qquad \mu\textrm{-a.e.\,.}
  \end{equ}
  Because the operator $T^\lambda$ is a contraction in $L^2(\mathcal S,\mu)$, this condition can only be satisfied if $V_{\nu} = V^*$ $\mu$-a.e..
\end{proof}
\noindent Consequently, suboptimal fixed points of the dynamics \eref{e:mfpde} cannot satisfy \eref{e:fixedpoint} Lebesgue-a.e. in $\Theta$.}

\subsubsection{Instability of local minima}\label{ss:2}
\edit{  We prove below that spurious local minima are avoided by the dynamics as discussed in the main text. More specifically, we lead the assumption of convergence of the dynamics to a suboptimal fixed point $\tilde \nu$ to a contradiction by combining it with  the approximate gradient structure of the \abbr{td} vector field when $\nu_t$ is close to one of such stationary points. This proof leverages two important facts: that by \lref{p:mf} suboptimal fixed points $\tilde \nu$ with $V_{\tilde \nu} \neq V^*$ imply the existence of regions in parameter space where the transport vector field does not vanish and that the solution to \eref{e:mfpde} does not lose (projected) support for any finite time, as summarized in \lref{l:support} below. These facts are finally combined in \lref{l:spurious}, leading the assumption of convergence to a suboptimal fixed point to the desired contradiction.}

    By the assumed structure of the approximator we note that all measures with the same expectation in the homogeneous component result in the same approximator, \ie if $\nu$, $\nu'$ are such that $\int \omega_0 \nu(\d \omega_0, \d \bar \omega) = \int \omega_0 \nu'(\d \omega_0, \d \bar \omega)$ a.e. then clearly
    \begin{equ}
      V_{\nu}(\,\cdot\,) = \int \omega_0 \phi(\,\cdot\,; \bar \omega ) \nu(\d \omega_0, \d \bar \omega) = \int \omega_0 \phi(\,\cdot\, ; \bar \omega) \nu'(\d \omega_0, \d \bar \omega) = V_{\nu'}(\,\cdot\,) \,.
    \end{equ}
    We therefore introduce the quantity
    \begin{equ}\label{e:h}
      h_\nu^1(\bar \omega) := \int \omega_0 \nu(\d \omega_0, \d \bar \omega)
      \end{equ} which will play an important role in the proof of \tref{t:mf} below.

\edit{We now proceed to state and prove the main lemmas of the section:}
    {\begin{lemma}\label{l:support}
    Let \aref{a:mf} a) hold and let $\nu_0$ satisfy \aref{a:mf} c), then $\nu_t$ solving \eref{e:mfpde} with initial condition $\nu_0$ satisfies \aref{a:mf} c) for every $t > 0$.
  \end{lemma}}
  Defining throughout $\|\,\cdot\,\|_{BL}$ as the bounded Lipschitz norm, we further prove that
    \begin{lemma}\label{l:spurious}
      Let \aref{a:mf} hold and $\tilde \nu$ be a fixed point of \eref{e:mfpde} such that \eref{e:fixedpoint} does not hold a.e.. There exists $\epsilon>0$ such that if $\|h_{\tilde \nu}^1-h_{\nu_{t_1}}^1\|_{BL}< \epsilon$ for a $t_1>0$ there exists $t_2> t_1$ such that $\|h_{\tilde \nu}^1-h_{\nu_{t_2}}^1\|_{BL} > \epsilon$.
    \end{lemma}

    {\begin{proof}\textbf{of \lref{l:support}}
      This result corresponds to \cite[Lemma C.13]{ChizatBach18} about the stability of the separation property \aref{a:mf} c) under pushforwards of the initial condition $\nu_0$ under the integrated gradient flow map $X_t$ defined as
      \begin{equ}\label{e:Xt}
        \partial_t X(t,u) = v_t(X(t,u))\qquad \text{with} \qquad X(0,\cdot) = \mathrm{id}
      \end{equ}
      where $\mathrm{id}~:~\Omega \to \Omega$ denotes the identity and $v_t$ is the vector field of the transport partial differential equation \eref{e:mfpde}.
      We note that in its original form \cite{ChizatBach18} this result does \emph{not} leverage the gradient flow structure of the dynamics, but only the \emph{continuity} of the map $X_t$, proven in \cite[Lemma B.4]{ChizatBach18} under assumptions \aref{a:mf} a) or, alternatively, \aref{a:mf2}. It is therefore sufficient for our purposes to establish such continuity property for the map $X_t$ in the present setting, \ie when $v_t$ is the vector field of the transport equation \eref{e:mfpde}. The continuity of the above map results immediately  from the one-sided Lipschitz property from \aref{a:mf} a) enjoyed by $v_t$ on the sets $Q_r = [-r,r]\times \Theta$ uniformly on compact time intervals, as mentioned in \cite[Proof of Lemma B.4]{ChizatBach18}. This regularity is in turn guaranteed by the Lipschitz continuity and Lipschitz smoothness of $\psi$ from \aref{a:mf} and bundedness of $r$.
    \end{proof}}

    \begin{proof}\textbf{of \lref{l:spurious}}
      By \lref{p:mf}, the key quantity
      \begin{equ}\label{e:g}
      g_{\tilde \nu}(\bar \omega) := \dtp{\partial_{\omega_0}\psi(\cdot; \omega), \delta({\tilde \nu})} = \dtp{\psi(\cdot; (1,\bar \omega)), \delta({\tilde \nu})} = \dtp{\phi(\cdot; \bar \omega), \delta({\tilde \nu})}
      \end{equ}
      cannot vanish a.e. on $\Theta$. Then, by \aref{a:mf} there exists a nonzero regular value $-\eta$ of $g_{\tilde \nu}(\bar \omega)$, which without loss of generality we assume to be negative, so that $\eta >0$ (else invert the signs of $\omega_0$ in the rest of the proof). 
      To conclude the proof we introduce the (nonempty) sublevel set $\mathcal A := \{(\omega_0,\bar \omega) \in \Omega~:~g_{\tilde \nu}(\bar \omega) < -\eta \}$ and define
      \begin{equ}\label{e:A}
      A_+ = \{(\omega_0,\bar \omega) \in \mathcal A~:~ \omega_0 >0\}\,.
      \end{equ}
        Further denoting by $\bar A \subseteq \Theta$ the projection of $\mathcal A$ onto $\Theta$, we note that the level set $\partial \bar A$ is an orientable manifold of dimension $m-2$ and is by definition orthogonal to the gradient field of $g_{\tilde\nu}(\bar \omega)$.
      By continuity of $\nabla g_{\tilde \nu}(\bar \omega)$ when $\bar A$ is compact (which is \eg automatic when $\Theta$ is bounded as in \aref{a:mf2}) there exists
      \begin{equ}
       \beta := \min_{\bar \omega \in \partial \bar A} n_{\bar \omega} \cdot \nabla_{\bar \omega} g_{\tilde \nu}(\bar \omega)  > 0 \, ,
     \end{equ} 
      where $n_{\bar \omega}$ denotes the normal unit vector to $\partial \bar A$ in the outward direction.\footnote{when $\bar A$ is not compact we must choose $\eta$ so that it is also a regular value of the function on $\{\bar \omega \in \Rr^{m-1}~:~\|\bar \omega\|_2=1\}$ to which $g$ converges as $\bar \omega$ goes to infinity.}

            We now choose $\epsilon$ small enough that assuming
            \begin{equ}\label{e:beh}
              \|h_{\tilde \nu}^1-h_{\nu_{t}}^1\|_{BL}< \epsilon \qquad \text{for all } t> t_1\,,
              \end{equ}
               leads to a contradiction. More specifically, by \lref{l:boundong} we choose $\epsilon(\alpha, \eta, \beta)$ small enough so that for all $\nu_{t}$ such that \eref{e:beh} holds we have $g_{\nu_{t}}(\bar \omega) < -\eta / 2$ on $\bar A$ and $n_{\bar \omega} \cdot\nabla_{\bar \omega} g_{\nu_{t}} > \beta/2$ on $\partial \bar A$.
            Then, the two inequalities above combined with $\partial_{\omega_0} \psi(\omega_0,\bar \omega) = \psi(1,\bar \omega)$ imply that the set $A_+$ defined above is forward invariant and therefore hat $\partial_t \nu_t(A_+) \geq 0$ as long as \eref{e:beh} holds.  Furthermore, by similar arguments we notice that no trajectories enter the set $\mathcal A \setminus A_+$ after $t_1$.

      We now distinguish two cases: either (i) $\nu_{t_1}(A_+)> 0 $ or (ii) $\nu_{t_1}(A_+) = 0 $. We treat these two cases separately by mimicking the arguments of the proofs of \cite[Lemmas C.4, C.18]{ChizatBach18}, respectively.
      \begin{enumerate}
        \item Assume that $\nu_{t_1}(A_+)> 0$. We note that, besides the forward invariance of the set $A_+$, under our assumptions the first component of the velocity field in $\mathcal A$ is lower bounded by $\eta/2$, so that $\omega_0(t) = \omega_0(0) + t \eta/2 $ is a subsolution to the $\omega_0$-component of the trajectory of a test mass with initial condition with $\omega(0) \in \mathcal A$, as long as $\bar \omega(t) \in \bar A$. In particular, by the forward invariance of $A_+$, if $\omega(0) \in A_+$ we have $\omega_0(t)>  t \eta/2 $.
        Therefore, assuming that $\text{supp}(\nu_t) \subset (-M,M)\times \Theta$ we must have for every $t>t_1$ 
        \begin{equ}
          h_{\nu_t}^1(\bar A) \geq \eta/2 (t-t_1) \nu_{t_1}(A_+) + \min\{0, (t-t_1)\eta/2  - M\}\nu_{t_1}(\mathcal A \setminus A_+)\,.
        \end{equ}
        In particular, for $t> t_1 + 2 M/\eta$ the above quantity grows linearly in time, contradicting the assumption made above that $\|h_{\tilde \nu}^1-h_{\nu_{t_1}}^1\|_{BL}< \epsilon$ for all $t> t_1$.
        \item Assume now that $\nu_{t_1}(A_+)= 0$. We show that there exists $t_2> t_1$ such that $\nu_{t_2}(A_+)> 0$, so that the proof is completed by applying part i) above from $t_2$. Indeed let $\omega^* \in \text{supp}(\nu_{t_1})$ be such that $\bar \omega^*\in \bar A$ is a local minimum of $g_{\tilde \nu}$, \ie for which $\nabla g_{\tilde \nu} = 0$. Then choosing $\tilde \epsilon$ such that $\BB_{\tilde \epsilon}(\bar \omega^*) \subset \bar A$, and choosing $M$ large enough such that $\supp(\nu_{t_1}) \subseteq [-M,M]\times \Theta$, we know by \lref{l:badcase} that there exists $t_2 > t_1$ such that the image at $t_2$ of $\omega(t_1):=\omega^*$ under the flow of the \abbr{td} vector field is contained in $A_+$. By continuity of pushforward map of such vector field, this must also hold for a neighborhood of $\omega^*$, to which $\nu_{t_1}$ assigns positive mass.
      \end{enumerate}
    \end{proof}
    Defining by $\|\,\cdot\,\|_{C^1}$ the maximum of the supremum norm of a
function and the supremum norm of its gradient and recalling the structure of the temporal difference vector field:
    \begin{equ}\label{e:mftdvf}
      v_t(\omega) = -\dtp{\nabla \omega_0\phi(\bar \omega),  \delta(s,s',\nu_t) P(s,\d s')\mu(\d s)} = - \nabla (\omega_0 \,g_{\nu_t}(\bar \omega))
    \end{equ}
    where $\nu_t$ solves \eref{e:mfpde} we are ready to state the lemma needed to prove part ii).

    \begin{lemma}\label{l:badcase}
    Let  $\tilde \nu \in \mathcal M_+(\Omega)$ and $\bar \omega^*$ satisfy $|\nabla g_{\tilde \nu}(\bar \omega^*)| = 0$, $g_{\tilde \nu}(\bar \omega^*) < -\eta < 0$ for some $\eta > 0$. Then for every $M, \tilde \epsilon>0$ there exists $\epsilon, t_2>0$ such that if for all $t \in (0, t_2)$ we have $\|g_{\tilde \nu}- g_{\nu_t}\|_{C^1} < \epsilon$ and $\omega_0^* \in [-M,0]$,
      then at time $t_2$ the point $\omega^*$ is mapped, under the flow of the TD vector field \eref{e:mftdvf} to a subset of $\BB_{\tilde \epsilon}((1,\bar \omega^*))$.
    \end{lemma}

\def\ws{\bar \omega^*}
    {\begin{proof}\textbf{of \lref{l:badcase}}
      Defining $q(t):= \|\bar \omega(t) - \bar \omega^*\|$, we see that the trajectory $(\omega_0(t), \bar \omega(t))$ of a particle under \eref{e:mftdvf} with initial condition $\omega(0) = \ws$ must satisfy:
      \begin{equ}
          \frac \d {\d t} \omega_0(t) = - g_{\nu_t}(\bar \omega(t)) \geq - g_{\nnu}(\bar \omega^*) - |g_\nnu(\bar \omega(t))- g_\nnu(\ws)| - |g_{\nu_t}(\bar \omega(t))- g_\nnu(\bar \omega(t))|
        \end{equ}
        and
        \begin{equs}
          \frac \d {\d t} q(t) & \leq |\omega_0(t)| \|\nabla_{\bar \omega} g_{\nu_t}(\bar \omega(t))\| \\&\leq   |\omega_0(t)| \pq{\|\nabla_{\bar \omega} g_{\nnu}(\bar \omega^*)\| + \|\nabla_{\bar \omega} g_\nnu(\bar \omega(t))- \nabla_{\bar \omega} g_\nnu(\ws)\| + \|\nabla_{\bar \omega}g_{\nu_t}(\bar \omega(t))- \nabla_{\bar \omega}g_\nnu(\bar \omega(t))\|}
      \end{equs}
      for all $t \in [0,\bar \tau]$ where $\bar \tau := \inf \{t~:~ \omega_0(t) \not \in [-M,1]\}$.
      Furthermore, by the Lipschitz continuity of $g_\nnu(\,\cdot\,)$ and its Lipschitz smoothness there exists $L>0$ such that $\max\{|g_\nnu(\bar \omega)- g_\nnu(\ws)|, \|\nabla g_\nnu(\bar \omega)- \nabla g_\nnu(\ws)\|\} \leq L \|\bar \omega-\ws\|$.
      Since by assumption $\|g_{\tilde \nu}- g_{\nu_t}\|_{C^1} < \epsilon$, for $t \in [0,\bar \tau]$ we must have
      \begin{equ}
        \begin{cases}
          \frac \d {\d t} \omega_0(t)  \geq \eta - \epsilon - L q(t)\\
          \frac \d {\d t} q(t)  \leq  |\omega_0(t)|\pq{\epsilon + L q(t)}
        \end{cases}
        \end{equ}

      Upon possibly increasing the value of $L$ such that $\eta/{4L}<\tilde \epsilon$, we define $\tau_q = \inf\{t~:~q(t)>\eta/{4L}\}$ and we now proceed to show that one can choose $\epsilon\in (0,\eta/4)$
      such that $\tau_q > \bar \tau$, \ie that the forward dynamics of the point $\omega^* = (\omega_0^*, \bar \omega^*)$ will reach $A_+$ before $q(t)>\tilde \epsilon$. Recall that on $[0, \tau_q]$ and for $\epsilon\in (0,\eta/4)$ we have
      \begin{equ}
        \omega_0(t) \geq \omega_0(0) + \frac \eta 2 t
      \end{equ}
      and in particular $\omega_0(t)>\omega_0(0) \geq -M$. This implies that for all $t \in [0,\bar \tau \wedge \tau_q]$ we have $\frac \d {\d t} q(t)< q(0)+M \epsilon + LM q(t)$, so that by Gronwall we obtain $q(t) \leq  \epsilon M \exp\pq{LMt}$. Finally, setting $\tau_0 := 2(M+1)/\eta \geq - 2(\omega_0(0)-1)/\eta>\bar \tau$ we note that we can choose $\epsilon$ small enough such that $\tau_q>\tau_0> \bar \tau$, \ie by monotonicity of $q(t)$ under our bound, such that
    \begin{equ}
      q(\tau_0) \leq \epsilon M \exp\pq{2LM(M+1))/\eta} \leq \eta /{4L}\,.
    \end{equ}
    Choosing $\epsilon\in (0, \eta/4)$ such that the \abbr{rhs} inequality holds concludes the proof.
    \end{proof}}

    \begin{lemma}\label{l:boundong} Recalling the definitions of $g_\nu$, $h_\nu^1$ from \eref{e:g}, \eref{e:h} respectively, for all $C_0 > 0$ there exists $\alpha>0$ and for all $\nu, \nu'$ satisfying $\|h_{\nu}^1\|_{BL}, \|h_{\nu'}^1\|_{BL} < C_0$, one has
      \begin{equ}
        \|g_{\nu} - g_{\nu'}\|_{C^1} \leq \alpha \|\phi\|_{C^1}^2 \|h_{\nu}^1 - h_{\nu'}^1\|_{BL}\,.
      \end{equ}
    \end{lemma}
    \begin{proof}\textbf{of \lref{l:boundong}}
      Recognizing that $\psi(s;(1,\bar \omega)) = \phi(s; \bar \omega)$
      and letting $\alpha$ be the Lipschitz constant of $\dtp{\,\cdot\,,\delta}$ on the bounded set $\{\int \psi \nu (\d \omega) ~:~ \|h_{\nu}^1\| < C_0\}$
      we have
      \begin{equs}
        \|g_{\nu} - g_{\nu'}\|_{C^1} & = \|\dtp{\psi(s;(1, \,\cdot\,)), \delta(\nu)}_\FF - \dtp{\psi(s;(1, \,\cdot\,)), \delta(\nu')}_{\FF}\|_{C^1} \\& = \|\dtp{\phi(s;\,\cdot\,), \delta(\nu) - \delta(\nu')}_\FF\|_{C^1} \leq  \alpha \|\phi\|_{C^1} \| \delta(\nu) - \delta(\nu')\|_{\FF}\\
        & \leq \alpha\|\phi\|_{C^1} \sup_{f \in \mathcal F, \|f\|=1} \int{\dtp{f,\phi} \d(h_{\nu}^1 - h_{\nu'}^1)}(\bar \omega)\\
        & \leq  \alpha\|\phi\|_{C^1}^2 \|h_{\nu}^1 - h_{\nu'}^1\|_{BL}
      \end{equs}
      where we used that $\dtp{f,\phi}_\FF$ is $\|\,\phi\,\|_{C^1}$-Lipschitz and upper-bonded by $\|\phi\|_{C^1}$ when $\|f\|_\FF<1$.
    \end{proof}

\end{document}